\documentclass[journal]{IEEEtran}
\usepackage{amsmath,amsfonts,amsthm,amssymb}
\usepackage{algorithmic}
\usepackage{array}
\usepackage{textcomp}
\usepackage{stfloats}
\usepackage{url}
\usepackage{verbatim}
\usepackage{graphicx}
\usepackage{xcolor}
\usepackage{pgfplots}
\pgfplotsset{compat=newest}
\usepgfplotslibrary{fillbetween, groupplots}
\usepackage{bm}
\usepackage{bbm}
\usepackage[font=small]{caption}
\usepackage{subcaption}
\usepackage{tikzscale}
\usetikzlibrary{%
	arrows,%
	arrows.meta,%
	automata,%
	calc,%
	chains, %
	patterns,%
	plotmarks,%
	positioning,%
	shapes.geometric,%
	shapes,%
	snakes,%
        spy,%
}
\usepackage{nicefrac}
\usepackage{diagbox}
\usepackage{stackengine}
\usepackage{booktabs}
\def\BibTeX{{\rm B\kern-.05em{\sc i\kern-.025em b}\kern-.08em
    T\kern-.1667em\lower.7ex\hbox{E}\kern-.125emX}}
\usepackage{balance}

\theoremstyle{plain}

\newtheorem{proposition}{Proposition}
\newtheorem*{proposition*}{Proposition}

\theoremstyle{definition}
\newtheorem{definition}{Definition}

\theoremstyle{remark}

\newtheorem{example}{Example}

\DeclareMathOperator*{\argmax}{arg\,max}
\DeclareMathOperator*{\argmin}{arg\,min}

\usepackage{academicons}
\newcommand{\orcid}[1]{\href{https://orcid.org/#1}{\textcolor[HTML]{A6CE39}{\aiOrcid}}}

\newcommand{\bfx}{\bm{x}}

\newcommand{\bfc}{\bm{c}}
\newcommand{\bfb}{\bm{b}}
\newcommand{\bfe}{\bm{e}}
\newcommand{\bfX}{\bm{X}}
\newcommand{\bft}{\bm{t}}

\newcommand{\bfp}{\bm{p}}

\newcommand{\bfs}{\bm{s}}
\newcommand{\bfv}{\bm{v}}
\newcommand{\bfu}{\bm{u}}
\newcommand{\bfd}{\bm{d}}

\newcommand{\bfL}{\bm{L}}
\newcommand{\bfA}{\bm{A}}
\newcommand{\calP}{\mathcal{P}}
\newcommand{\calX}{\mathcal{X}}
\newcommand{\calR}{\mathcal{R}}
\newcommand{\calD}{\mathcal{D}}
\newcommand{\calM}{\mathcal{M}}

\newcommand{\msv}{\mathsf{v}}
\newcommand{\msc}{\mathsf{c}}
\newcommand{\hbfd}{\hat{\bm{d}}}
\newcommand{\hd}{\hat{d}}
\newcommand{\msFA}{\mathsf{FA}}
\newcommand{\msMD}{\mathsf{MD}}

\newcommand{\rowspan}{\ensuremath{\mathsf{Sp}_{\text{r}}}}

\newcommand{\nm}{n_{\mathsf{m}}}
\newcommand{\nmest}{\hat{n}_{\mathsf{m}}}
\newcommand{\nmmax}{n_{\mathsf{m}}^{\max}}
\newcommand{\rvnm}{N_{\mathsf{m}}}
\newcommand{\transpose}{^\mathsf{T}}
\newcommand{\bfa}{\bm{a}}
\newcommand{\tbfd}{\tilde{\bm{d}}}
\newcommand{\td}{\tilde{d}}
\newcommand{\tbfs}{\tilde{\bm{s}}}
\newcommand{\ts}{\tilde{s}}

\newcommand{\APP}{\mathsf{APP}}
\newcommand{\calE}{\mathcal{E}}

\newcommand{\setpoints}{\mathcal{C}}

\newcommand{\points}{\bfc}

\newcommand{\numclusters}{k}
\newcommand{\numclustersest}{\hat{k}}
\newcommand{\numclustersmax}{\numclusters_{\max}}
\newcommand{\silhouette}{\mathsf{s}}
\newcommand{\silhouettethres}{\mathsf{s}^{\text{thres}}}
\newcommand{\intradis}{\mathsf{a}}
\newcommand{\interdis}{\mathsf{b}}
\newcommand{\dunn}{\mathsf{d}}
\newcommand{\pca}{p}
\newcommand{\vecpca}{\bfp}
\newcommand{\norm}[1]{\left\lVert#1\right\rVert}

\definecolor{darkblue}{rgb}{0.07843,0.16863,0.54902}
\definecolor{darkgreen}{rgb}{0,0.49804,0}%
\definecolor{brown}{rgb}{0.85098, 0.32941, 0.10196}%

\newcommand{\FL}{FL}

\newcommand{\NAME}{FedGT}

\begin{document}
\title{FedGT: Identification of Malicious Clients in Federated Learning with Secure Aggregation}

\author{\IEEEauthorblockN{Marvin Xhemrishi, \emph{Student Member, IEEE}, Johan \"Ostman, Antonia Wachter-Zeh \emph{Senior Member, IEEE},\\  and Alexandre Graell i Amat, \emph{Senior Member, IEEE} \vspace{-.5cm}} 
	\thanks{M. Xhemrishi and A. Wachter-Zeh are with School of Computation, Information and Technology, Technical University of Munich, 80333 Munich, Germany~(e-mails: {\{marvin.xhemrishi, antonia.wachter-zeh\}@tum.de}).} 

\thanks{J. \"Ostman is with AI Sweden, Lindholmspiren 11, 41756 Gothenburg, Sweden~(e-mail: {johan.ostman@ai.se}).} 

 	\thanks{A. Graell i Amat is with the Department of Electrical Engineering, Chalmers University of Technology, 41296 Gothenburg, Sweden~(e-mail: { alexandre.graell@chalmers.se}).} 
 
	\thanks{This work was partially supported by the German Research Foundation (DFG) under Grant Agreement No. WA 3907/7-1, by the Swedish Innovation Agency (VINNOVA) under grant 2021-04783, by the Swedish Research Council (VR) under grant 2020-03687, and by the Wallenberg AI, Autonomous Systems and Software Program (WASP) funded by the Knut and Alice Wallenberg Foundation.}
}

\maketitle

\begin{abstract}
We propose FedGT, a novel framework for identifying malicious clients in federated learning with secure aggregation. 
Inspired by group testing, the framework leverages overlapping groups of clients to identify the presence of malicious clients in the groups via a decoding operation.
The clients identified as malicious are then removed from the model training, which is performed over the remaining clients. 
By choosing the size, number, and overlap between groups, FedGT strikes a balance between privacy and security.
Specifically, the server learns the aggregated model of the clients in each group\textemdash vanilla federated learning and secure aggregation correspond to the extreme cases of FedGT with group size equal to one and the total number of clients, respectively.
The effectiveness of FedGT is demonstrated through extensive experiments on the MNIST, CIFAR-10, and ISIC2019 datasets in a cross-silo setting under different data-poisoning attacks. 
These experiments showcase FedGT's ability to identify malicious clients, resulting in high model utility. 
We further show that FedGT significantly outperforms the private robust aggregation approach based on the geometric median recently proposed by Pillutla \emph{et al.} in multiple settings.
\end{abstract}

\begin{IEEEkeywords}
Federated learning, group testing, malicious clients, poisoning attacks, secure aggregation, security, privacy.
\end{IEEEkeywords}

\section{Introduction}
\IEEEPARstart{F}{ederated} learning (FL) \cite{McM17} is a distributed machine learning paradigm that enables multiple devices (clients) to collaboratively train a machine learning model under the orchestration of a central server.  To preserve data privacy,  clients share their locally-trained models with the central server instead of their raw data. 

In its original form, {\FL} is susceptible to  model-inversion attacks \cite{Fre15,Wan19}, which allow the central server to infer clients' data from their local model updates. As demonstrated in~\cite{Dimitrov22}, such attacks can be mitigated by employing secure aggregation protocols \cite{Bon17,Bel20}. 
These protocols guarantee that the server only observes the aggregate of the client models instead of individual models. 

A salient problem in {\FL} is poisoning attacks \cite{Bar19}, where malicious and/or faulty clients corrupt the jointly-trained global model by introducing
mislabeled training data (\emph{data poisoning}) \cite{Tol20,Wan20}, or by modifying local model updates (\emph{model poisoning}) \cite{Fung18}.  
Poisoning attacks pose a serious security risk for critical applications.
Defensive measures against these threats generally fall into two categories: robust aggregation and anomaly detection. 
Robust aggregation techniques~\cite{Bla17, Yin2018, Cao19} are reactive approaches designed to mitigate the effect of poisoned models, whereas anomaly detection is inherently proactive and aims to identify and eliminate corrupted models~\cite{Li20, Mallah21, nguyen22}.
Robust aggregation techniques can introduce bias, especially when clients have heterogeneous data~\cite{Li20}, and their effectiveness tends to diminish with an increasing number of malicious clients~\cite{zhang22det}.
Moreover, a recurring issue with defense mechanisms is their reliance on accessing individual client models, leaving clients vulnerable to model-inversion attacks.
Addressing resiliency against poisoning attacks and devising protocols for the identification of  malicious clients in {\FL} without access to individual client models remains a challenge~\cite{Kai21, gong23}.
Notably, privacy-preserving techniques such as secure aggregation enhance clients' privacy at the expense of camouflaging adversaries~\cite{Kai21}. 
Hence, there is a fundamental trade-off between security and privacy.

{\emph{Contribution:}} In this paper, we propose FedGT, a novel framework for identifying  malicious clients in {\FL} with secure aggregation. Our framework is inspired by  group testing \cite{Dor43}, 
a paradigm  to identify defective items in a large population that significantly reduces the required number of tests compared to the naive approach of testing each item individually.
FedGT's key idea is to group clients into overlapping groups. 
For each group, the central server observes the aggregated model of the clients  
and runs a suitable test to identify the presence of malicious clients in the group. 
The malicious clients are then identified through a decoding operation at the server, allowing for their removal from the training of the global model. %
FedGT trades-off client's data privacy, provided by secure aggregation, with \emph{security}, understood here as the ability to identify malicious clients. It encompasses both non-private vanilla FL and privacy-oriented methods, e.g., secure aggregation, by selecting group sizes of one and the total amount of clients, respectively. 
However, by allowing group sizes between these two extremes, FedGT strikes a balance between privacy and security, i.e., improved identification capability comes at the cost of secure aggregation involving fewer clients.
Moreover, FedGT does not require any hyper-parameter tuning.

We showcase FedGT's effectiveness in identifying malicious clients through experiments on the MNIST, CIFAR-10, and ISIC2019 datasets under both targeted and untargeted data-poisoning attacks. 
Our focus is specifically on the cross-silo scenario, wherein the number of clients is moderate (up to {$50$~\cite{Ogier2022flamby}}) and data-poisoning is the predominant attack vector~\cite{Shej21}. Fig.~\ref{fig:mal_vs_att_acc} shows the performance of two proposed versions of FedGT as a function of the  numbers of malicious clients, $\nm$, for a scenario with 15 clients and a targeted label-flipping attack. The results  correspond to a classification problem over the CIFAR10 dataset after 30 communication rounds. 
When no defense mechanism is in place, the attack success significantly increases as the number of malicious clients grows.
Remarkably, FedGT enables the identification and removal of malicious clients with low misdetection and false alarm probabilities. This leads to a substantially reduced attack accuracy\textemdash significantly outperforming the recently-proposed %
robust federated aggregation (RFA) protocol based on the geometric median~\cite{Pillutla22}\textemdash  while achieving a lower communication complexity.%

\section{Related work}\label{sec:rel_work}

To the best of our knowledge, only the works \cite{So21,  Pillutla22, Zhang21} address resiliency against poisoning attacks in conjunction with secure aggregation. The work  \cite{So21} is the first single-server solution to account for both privacy and security in {\FL}.
The protocol is based on drop-out resilient secure aggregation where the server utilizes secret sharing to first obtain the pairwise Euclidean distance between the clients' updates and then selects what clients to aggregate by means of multi-Krum~\cite{Bla17}. However, it is not clear if the pairwise differences can leak extra information. In \cite{Pillutla22}, a robust aggregation protocol, dubbed RFA, is proposed. This protocol is based on an approximate geometric median, computed by means of secure aggregation. However, RFA lacks the capability to identify malicious clients and is known to be inferior to other robust aggregation techniques, especially when dealing with heterogeneous client data~\cite{Li23}. 
The work by \cite{Zhang21} presents a privacy-preserving tree-based robust aggregation method. 
In particular, each leaf in the tree consists of a subgroup of clients who securely aggregate their local models.
To achieve privacy between subgroups, masking is done on all but the last parameters in the aggregated models.
By using the Euclidean distance between the unmasked parameters and the corresponding parameters in the global model, an outlier removal scheme, based on variance thresholding, is used iteratively to determine what groups should contribute to the global model. The approach in~\cite{Zhang21} is the method closest to ours as it relies on dividing clients into subgroups and on testing the group aggregates. 
However, contrary to FedGT, it is unable to identify malicious clients and to leverage the information of overlapping groups.

\section{Preliminaries}\label{sec:prel}
\subsection{Notation}

We use lowercase bold letters and uppercase bold letters to denote row vectors and matrices, respectively,  e.g., $\bfx$ and $\bfX$. The $i$-th element of vector $\bfx$ is denoted as $x_i$.  %
We use calligraphic letters  to denote sets, e.g., $\calX$. For an integer $x$, we use the notation $[x]$ to denote the set of all positive integers less than equal to $x$, i.e., $[x] = \{1, 2, \dots, x\}$. The empty set is denoted by $\varnothing$. The logical disjunction operator is represented by $\vee$ and the logical conjunction operator  by $\land$. For a matrix $\bfX$ and a vector $\bfx$, we use the notation $\bfx\vee \bfX\transpose$ to denote a matrix-vector operator, similar to the multiplication, where the dot product is performed using the logical conjunction and the addition is computed using the logical disjunction. Finally, we denote by $w_\mathsf{H}(\bfx)$ the Hamming weight of vector $\bfx$, i.e., the number of nonzero entries of $\bfx$.
\begin{figure}[t!]
\resizebox{.9\linewidth}{!}{\definecolor{chocolate}{rgb}{0.48, 0.25, 0.0}
\definecolor{amber}{rgb}{1.0, 0.75, 0.0}
\begin{tikzpicture}

\begin{groupplot}[ 
    group style={
        group size=1 by 1,
   },
]     
    \nextgroupplot
    [
        grid = both, 
        grid style ={dotted,draw=black!90},
        tick label style ={/pgf/number format/fixed},
        minor tick num=3, 
        xmode=linear, 
        ymode=linear, 
        ymax = 0.4, 
        ymin = 0, 
        xmax = 5, 
        xmin = 0, 
        xlabel= \footnotesize $\nm$, 
        ylabel=\footnotesize attack accuracy,
        legend style =
        {
             legend columns=1,
             at={(axis cs:1,0.32)},
             fill=white,
             draw=black,
             anchor=center,
             legend cell align=left,
             align=left
        }
    ]
    
     \addplot[brown, ultra thick, solid, mark=pentagon*, mark options={fill=white} ,mark size=2.5pt] table [x index = {0}, y index={2}, col sep=comma]{./plots/csv_files/m_vs_att_acc/CIFAR10_targeted_attack_acc_vs_m.csv};
    \addlegendentry{\footnotesize no defense}

    \addplot[chocolate, ultra thick, solid, mark=diamond*, mark options={fill=white} ,mark size=2.7pt] table [x index = {0}, y index={4}, col sep=comma]{./plots/csv_files/m_vs_att_acc/CIFAR10_targeted_attack_acc_vs_m.csv}; 
    \addlegendentry{\footnotesize RFA~\cite{Pillutla22}}
    
    \addplot[red, ultra thick, mark=*, mark options={fill=white}, mark size=2pt] table [x index = {0}, y index={3}, col sep=comma]  {./plots/csv_files/m_vs_att_acc/CIFAR10_targeted_attack_acc_vs_m.csv}; 
    \addlegendentry{\footnotesize \NAME-$\nmest$}

    \addplot[amber, ultra thick, mark=triangle*, mark options={fill=white}, mark size=3pt] table [x index = {0}, y index={1}, col sep=comma]{./plots/csv_files/m_vs_att_acc/CIFAR10_targeted_attack_acc_vs_m_DELTA.csv};
    \addlegendentry{\footnotesize {\NAME}-$\Delta$}
    
    \addplot[darkblue, ultra thick, solid, mark=square*, mark options = {fill = white}, mark size=2pt] table [x index = {0}, y index={1}, col sep=comma]{./plots/csv_files/m_vs_att_acc/CIFAR10_targeted_attack_acc_vs_m.csv};
    \addlegendentry{\footnotesize oracle}
    
\end{groupplot}

\end{tikzpicture}%
}
\vspace{-1.5ex}
\caption{Attack accuracy on a cross-silo setting with $n=15$ clients on the CIFAR10 dataset for varying number of malicious clients, $\nm$, conducting a label-flip targeted attack.  }
    \label{fig:mal_vs_att_acc}
    \vspace{-3ex}
\end{figure}

\subsection{Group Testing} %
Group testing~\cite{Dor43,Ald19} encompasses a family of test schemes aiming at identifying items affected by some particular condition, usually called \emph{defective} items (e.g., individuals infected by a virus), among a large population of $n$ items (e.g., all individuals). The overarching goal of group testing is to design a testing scheme such that the number of tests needed to identify the defective items is minimized. The principle behind group testing is that, if the number of defective items is significantly smaller than $n$, then negative
tests on groups (or pools) of items can spare many individual tests. Following this principle,  items are grouped into overlapping groups, and tests are performed on each group. Based on the test results on the groups, the defective items can then be identified\textemdash in general with some probability of error\textemdash via a decoding operation.

\subsection{Threat Model}
We consider a cross-silo scenario with an honest-but-curious server and $n$ clients out of which $\nm$ are compromised (referred to as \emph{malicious} clients).
The number  of malicious clients and their identities are unknown to the server.

A client may be compromised due to hardware malfunction or adversarial corruption.
In the latter case, we assume that the malicious clients can collude and perform coordinated attacks against the global model.
In this paper, we focus on data-poisoning attacks, which constitute the most realistic type of attack for cross-silo {\FL} \cite{Shej21}.%

\section{\NAME: group testing for {\FL}\\ with secure aggregation}\label{sec:FedGT}

We consider a population of $n$ clients, $\nm$ of which are malicious. We define the \emph{defective vector} $\bfd=(d_1, d_2, \dots, d_n)$ with entries representing whether a client $j$ is malicious ($d_j = 1$) or not ($d_j = 0$). %
It follows that $\sum_{j=1}^n d_j=\nm$.  Note that $\bfd$ is unknown, i.e., we do not know a priori which clients are malicious.

\begin{figure}[!t]
    \centering
    \includegraphics[width = .5\linewidth]{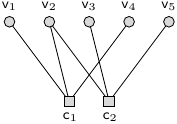}
    \caption{The bipartite graph of the matrix $\bfA$ in Example~\ref{ex:ExampleBipartiteGraep}. The circles represent variable nodes and  the squares represent check nodes.}
    \label{fig:gt}
    \vspace{-3ex}
\end{figure}

Borrowing ideas from group testing \cite{Dor43}, the $n$ clients are grouped into $m$ overlapping \emph{test groups}. We denote by $\calP_i$ the set of client indices belonging to test group $i\in[m]$, i.e., if client~$j$ is a member of test group $i$, then $j \in \calP_i$.  
\begin{definition}[Assignment matrix]
The assignment of clients to test groups can be described by an assignment matrix $\bfA=(a_{i,j})$, $i\in[m]$, $j\in[n]$, where $a_{i,j}=1$ if client $j$ participates in test group $i$ and $a_{i,j}=0$ otherwise.
\end{definition}

The assignment of clients to test groups, i.e., matrix $\bfA$, can also be
conveniently represented by a bipartite graph consisting of $n$ \emph{variable nodes} (VNs) $\msv_1,\ldots,\msv_n$  corresponding to the $n$ clients, and $m$ \emph{constraint nodes} (CNs) $\msc_1,\ldots,\msc_m$, corresponding to the $m$ test groups. An edge between VN $\msv_j$ and CN $\msc_i$ is then drawn if client $j$ participates in test group $i$, i.e., if $a_{i,j}=1$. 
Matrix $\bfA$\textemdash and hence the corresponding bipartite graph\textemdash is a design choice that may be decided offline, analogous to the model architecture, and shared with the clients for transparency.

\begin{example}
\label{ex:ExampleBipartiteGraep}
The bipartite graph corresponding to a scenario with $5$ clients and $2$ test groups with assignment matrix
\begin{align*}
    \bfA = \begin{pmatrix}
    1 & 1 & 0 & 1 & 0 \\
    0 & 1 & 1 & 0 & 1
\end{pmatrix}
\end{align*}
is depicted in Fig.~\ref{fig:gt}.
\end{example}

In FedGT, for each test group, a secure aggregation mechanism is employed to reveal only the aggregate of the client models in the test group to the server. 
Let $\bfu_i$, $i\in[m]$, be the aggregate model of test group $i$. 
For each test group $i$,
the central server applies a binary test on the corresponding aggregate model, $\mathsf{t}: \bfu_i \rightarrow \{0,1\}$. %
Let $t_i=\mathsf{t}(\bfu_i)\in\{0,1\}$ be the result of the test for test group $i$, where $t_i=1$ if the test is positive, i.e., there is at least a malicious client in the group, and $t_i=0$ if the test is negative, i.e., no malicious clients are in the group. 
We collect the result of the $m$ tests into the binary vector $\bft = (t_1, t_2, \dots, t_m)$.

We propose a suitable test in Section~\ref{sec:TestDesign}. However, we remark that the proposed framework is general and can be applied to any test on the test group aggregates.

We define the syndrome vector $\bfs=(s_1,\ldots,s_m)$, where $s_i=1$ if at least one client participating in test group $i$ is malicious and $s_i=0$ if no client participating in test group $i$ is malicious, i.e.,
\begin{align}
\label{eq:syndrome}
s_i=\bigvee_{j \in \calP_i} d_j,
\end{align}
and
\begin{align}
\label{eq:syndromevector}
\bfs=\bfd\vee \bfA\transpose\,.
\end{align}

For perfect (non-noisy) test results, it follows that $\bft=\bfs$. However, note that the result of a test may be erroneous, i.e., the result of the test may be $t_i=1$ even if no malicious clients are present (i.e., $s_i=0$) or $t_i=0$ even if malicious clients are present (i.e., (i.e., $s_i=1$). In general, the (noisy) test vector $\bft$ is statistically dependent on the syndrome vector $\bfs$ according to an (unknown) probability distribution $Q(\bft|\bfs)$.

Given the test results $\bft$ and the assignment matrix $\bfA$, the goal of \NAME{} is to identify the malicious clients, i.e., infer the defective vector $\bfd$. 
The design of the assignment matrix $\bfA$ and the corresponding inference problem is akin to an error-correcting coding problem, where the assignment matrix $\bfA$ can 
seen as the parity-check matrix of a code, and the inference problem corresponds to a decoding operation based on $\bfA$ and $\bft$. Thus, a suitable choice for $\bfA$ is the parity-check matrix of a powerful error-correcting code, i.e., with good distance properties. Furthermore, $\bfd$ can be inferred by applying conventional decoding techniques. 
We denote by $\hbfd=(\hd_1,\ldots,\hd_n)$ the estimated defective vector provided by the decoding operation, and define $\hat{\calM}=\{i:\hd_i=1\}$. 
Once $\hbfd$ has been obtained, clients $i\in\hat{\calM}$ are excluded from the training and the server aggregates the models of the remaining\textemdash flagged non-malicious\textemdash clients by means of secure aggregation.

The performance of FedGT, measured in terms of the utility of the model, is affected by two quantities: the misdetection probability, i.e., the probability that a malicious client is flagged as non-malicious, and the false-alarm probability, i.e., the probability that a non-malicious client is flagged as malicious, defined as\footnote{As in other works, we use the normalization factor $1/n$. Hence,  $P_{\msMD}$ and $P_{\msFA}$ are not strictly probabilities, as their values do not lie between $0$ and $1$.}%
\begin{align*}
P_{\msMD}&\triangleq\frac{1}{n}\sum_{i=1}^n \Pr(\hd_i=0|d_i=1)\,, \\
P_{\msFA}&\triangleq\frac{1}{n}\sum_{i=1}^n \Pr(\hd_i=1|d_i=0)\,.
\end{align*}
A high misdetection probability will result in many malicious clients poisoning the global model, hence yielding poor utility, while a high false-alarm probability will result in excluding many non-malicious clients from the training, thereby also impairing the utility. The misdetection and false-alarm probabilities  depend in turn on the assignment matrix $\bfA$, the decoding strategy, and the nature of the test performed. 
We discuss the decoding strategy to estimate $\bfd$ in Section~\ref{sec:Decoding}. 

\subsection{Privacy-Security Trade-off}
\label{subsec:trade_off}

Vanilla {\FL} \cite{McM17} and  {\FL} with \emph{full} secure aggregation \cite{Bon17} can be seen as the two extreme cases of FedGT, corresponding to $n$ (non-overlapping) groups and a single group with $n$ clients, respectively.
In vanilla {\FL}, tests on individual models can be conducted, facilitating the identification of malicious clients. However, this comes at the expense of clients' privacy. In contrast, full secure aggregation provides privacy by enabling the server to observe only the aggregate of the $n$ models, but it does not permit the identification of malicious clients. 

Under {\NAME}, the server observes $m$ aggregated models $\bfu_1\ldots,\bfu_m$, with $\bfu_i = \sum_{j=1}^{n} a_{i,j}\bfc_j$ and $\bfc_j$ being the local model of client $j$.  The privacy of the clients increases with the number of models aggregated~\cite{Elk22}. Hence, {\NAME} trades privacy for providing security, i.e., identification of malicious clients. 
Furthermore, there might be additional privacy loss due to the aggregates being from overlapping groups. 
This loss depends on the assignment matrix $\bfA$ and is agnostic to the number of malicious clients participating in the training. The privacy of FedGT is given in the following proposition.

\begin{proposition}\label{prop:privacy}
Let the assignment of clients to test groups be defined by assignment matrix $\bfA$ and let $r$ be the smallest non-zero Hamming weight of the vectors in the row span of $\bfA$ (in the coding theory jargon, the minimum Hamming distance of the code generated by $\bfA$ as its generator matrix). Then  FedGT achieves the same privacy as a secure aggregation scheme with $r$ clients. 
\end{proposition}
\begin{proof}
Due to the overlapping groups arising from matrix $\bfA$, there might exist a vector $\bfb \in \mathbb{R}^m$ such that $\sum_{i=1}^m b_i\bfu_i = \bfc_j$  for some $j \in [n]$, or equivalently 
$\bfb\bfA = \bfe_j$, where $\bfe_j$ is the $j$-th unit vector.
This will occur if $\bfe_j \in \rowspan(\bfA)$, where $\rowspan(\bfA)$ is the row span of $\bfA$. Generally speaking, for a subset $\calR \subset [n]$ of cardinality $r$, if $\sum_{\iota\in \calR} f_{\iota}\bfe_{\iota} \in \rowspan(\bfA)$ where $f_{\iota} \neq 0$, there exists a vector $\bfb'$ such that  
$\sum_{i=1}^m b'_i\bfu_i = \sum_{\iota \in \calR}f_{\iota}\bfc_{\iota}$.
Thus, we conclude that {\NAME} achieves the same privacy as a secure aggregation scheme with $r< n$ clients, where $r$ is the smallest non-zero cardinality of the subset $\calR$. In other words, $r$ is the smallest non-zero Hamming weight of the vectors in the row span of $\bfA$.%
\end{proof}

\subsection{The Choice of Assignment Matrix $\bfA$}\label{subsec:limits_nm}

The assignment matrix $\bfA$ should be carefully chosen to balance the trade-off between privacy and security: To improve the identification of malicious clients, one should choose $\bfA$ as the parity-check matrix of an error-correcting code with good distance properties, while to achieve a high privacy level,  $\bfA$ should correspond to the generator matrix of a code of large minimum Hamming distance.
Such codes are readily available, see \cite{Mac77}.

On the other hand, for FedGT to effectively detect malicious clients with a low probability of false alarm, it is essential that some group tests yield negative results: If all tests are positive, i.e., $\bft=\bf{1}$, FedGT will  flag all clients as malicious, resulting in the highest (unnormalized) probability of false alarm of $P_{\msFA}=\frac{n-\nm}{n}$.

As the number of malicious clients grows, the likelihood of observing only positive test outcomes, i.e.,  $\bft = \bf{1}$, increases.   {Furthermore}, the choice of $\bfA$ highly impacts the probability of having all groups contaminated. {More precisely, the probability of having all groups contaminated is fully determined by  $\bfA$ and $\nm$. For small matrices $\bfA$, this probability can be computed exactly, while for larger ones, it can be approximated using a Monte Carlo approach.}

{The assignment matrix $\bfA$ should  be chosen such that the probability  of having all groups contaminated is small (note that $\nm$ is out of control of the designer). Alternatively, one can impose a constraint on the probability of all groups being contaminated and find the assignment matrix $\bfA$ that supports the maximum number of malicious clients, $\nmmax$, such that this constraint is satisfied.} The value $\nmmax$ is intrinsic to $\bfA$ and can be obtained offline, as outlined next.

Consider {the best-case scenario of  noiseless tests, i.e., $\bft = \bfs$ and let the $\nm$ malicious clients  be assigned uniformly at random.  Let $S_i$ be the random variable corresponding to the syndrome of the $i$-th group (corresponding also to the test outcome of the $i$-th group for noiseless group testing) and $\bm{S}=(S_1,\ldots,S_m)$ (the corresponding realizations are defined in \eqref{eq:syndrome} and \eqref{eq:syndromevector}).}
Then, for a fixed assignment matrix $\bfA$, we can solve
\begin{equation}\label{eq:prop_hist}
\nmmax = \argmax_{\nm}\; \lbrace \Pr(\bm{S}=\bm{1} \lvert \rvnm = \nm) \leq \kappa\rbrace\,,
\end{equation}
where $N_{\mathsf{m}}$ is the random variable that represents the number of malicious clients and $\kappa$ denotes the probability constraint of all groups being contaminated. 
This approach provides a systematic procedure of identifying assignment matrices that are suitable for a given scenario.

\subsection{Test Design} 
\label{sec:TestDesign}

{FedGT can be applied to any test on the test group aggregates. However, it is essential to design an accurate test, as the performance of FedGT is impaired by the noisiness of the test. In this section, we propose a test that, as shown in the numerical results section, yields low error rate.}

To begin, we make the observation that the utility of an aggregated model tends to decrease as it gets contaminated by a larger number of poisoned models. Moreover, the work in~\cite{tolpegin2020data} shows that when using dimensionality reduction tools like principal component analysis (PCA), the local models of the malicious clients tend to cluster around similar values, even in the first component.\footnote{The first component  captures the highest variance of the original observations (models).} Empirically, we observed that the findings from~\cite{tolpegin2020data} apply also to our group testing scenario: the aggregated models of  test groups with the same number of malicious clients tend to cluster around the same value in the first component. 

Motivated by this observation, we propose a testing strategy in which we first cluster the test groups (i.e., the corresponding aggregated models) into clusters based on the number of malicious clients involved in the test group. Then, for each cluster, we  compute the average utility of the aggregated models using a small validation dataset at the server (a minor assumption as motivated in Section~\ref{sec:experiments}), and finally, we declare the result of the test for the test groups within the cluster with highest average utility as negative ($t=0$) and the result of the test  for all other test groups as positive ($t=1$). 
The details of the proposed testing strategy are outlined below.

Let $v_i$ denote a measured utility metric of the aggregated model of test group $i$ evaluated on the validation dataset, and let $\bfv = (v_1, v_2, \dots, v_m)$. Also, let $\pca_i$ be the first principal component representation of the aggregated model of test group $i$ evaluated on the fully-connected layer, and let $\vecpca =\left(\pca_1, \dots, \pca_m\right)$.
We form $m$ points $\points_i = \left(v_i, p_i\right)$ and cluster them using the $k$-means algorithm~\cite{Llo82}. 
Since $k$-means requires the number of clusters $\numclusters$, we compute the maximum possible number of clusters, 
\begin{equation}
    \label{eq:max_clusters}
    \numclustersmax = \min \left\{m,\; \max_{i \in [m]} \; \lvert \calP_i\lvert \, +\, 1 \right\} \,, 
\end{equation}
and perform $k$-means clustering for all $\numclusters \in [\numclustersmax]$.

The next step is to determine the optimal number of clusters. Two popular metrics for this purpose are the Silhouette score~\cite{Rou87} and the Dunn index~\cite{Dun74}. 
In scenarios like ours,  where the number of points is relatively small compared to the number of clusters, the Dunn index tends to perform better than the Silhouette score. However, the Dunn index is not effective at determining if the data should be clustered into a single cluster. To address these limitations, we propose a combined approach: first, we use the Silhouette score to assess whether the data should be clustered into one  or several clusters. If multiple clusters are indicated,  we then  use the Dunn index to determine the precise number of clusters.

For a data point $\points_i$ in cluster $\setpoints_u$, the Silhouette score is defined as 
\begin{equation*}
\silhouette_i = \begin{cases}
    \dfrac{\interdis_i - \intradis_i}{\max\{\intradis_i, \interdis_i\}}, & \lvert \setpoints_u \lvert > 1, \\[2ex]
    0, & \lvert \setpoints_u \lvert = 1,
\end{cases} %
\end{equation*}
where $\interdis_i$ is the smallest mean  distance of $\points_i$ to all points in any other cluster,  
\begin{equation*}
    \interdis_i = \min_{u' \neq u}\dfrac{1}{\lvert \setpoints_{u'} \lvert} \sum_{\points_j \in \setpoints_{u'}} \norm{\points_i - \points_j}_2^2,
\end{equation*}
and $\intradis_i$ is the mean  distance of $\points_i$ to all other points in the same cluster, 
\begin{equation*}
    \intradis_i = \dfrac{1}{\lvert \setpoints_u \lvert -1} \sum_{\points_j \in \setpoints_u, j \neq i}\norm{\points_i - \points_j}_2^2 \,.
\end{equation*}
{For each $k$, the  Silhouette score of the corresponding clustering result, denoted as $\silhouette^{(k)}$, is then  the average of the individual Silhouette scores, i.e.,
\begin{align*}
\silhouette^{(k)} = \frac{1}{m}\sum_{i \in [m]} \silhouette_i\,.
\end{align*}}
The Dunn index~\cite{Dun74} is defined as 
\begin{equation*}
    \dunn_k = \frac{\min\limits_{i\in[k],j\in[k]: i\neq j} \norm{\bar{\points}_i- \bar{\points}_j}^2_2}{\max\limits_{u \in [k]} \max\limits_{\points_{i}, \points_{j} \in \setpoints_u} \norm{\points_{i} - \points_{j}}^2_2  }\,,
\end{equation*}
where $\bar{\points}_i$ denotes the center point of cluster $\setpoints_i$.

Based on the Silhouette score and the Dunn index, we determine the number of clusters as
\begin{equation}
\label{eq:hatk}
    \numclustersest = \argmax_{i\in[k]}\{\dunn_i\} \mathbbm{1}\{\silhouette_{\mathrm{max}} \geq \silhouettethres \} + \mathbbm{1}\{\silhouette_{\mathrm{max}} < \silhouettethres \}\,,
\end{equation}
where $\silhouette_{\mathrm{max}} = \max_{k\in[k_{\max}]}\silhouette^{(k)}$.
That is, we first threshold the Silhouette score to decide whether the datapoints $\{\points_i\}$ should be clustered into  one or several clusters and, if multiple clusters are suggested, we identify the precise number of clusters using the Dunn index (selecting  the number of clusters that maximizes the Dunn index). 

We use the outcome of the clustering to determine the test result for each test group. Note that each cluster corresponds to a different number of malicious clients. In this paper, however, we consider binary test results, i.e., for a given test group the test is positive ($t=1$) if the test determines that there is at least one malicious client in the group and the test is negative ($t=0$) if there is none. Hence, we only need to distinguish between clusters corresponding to test groups with no malicious clients and clusters corresponding to test groups with malicious clients. To this aim, for every cluster $\setpoints_i$, $i\in[\numclustersest]$, we compute the average utility as $\Bar{v}_i = \dfrac{1}{\lvert \setpoints_i \lvert} \sum\limits_{j : \points_j \in \setpoints_i} v_j$ and apply the decision rule
\begin{align*}
    t_j = \begin{cases}
      0 &  j : \points_j \in \setpoints_\iota \\
      1 & j : \points_j \notin \setpoints_\iota
    \end{cases} \,,\qquad \text{where  }
    \iota = \argmax_{i \in [\numclustersest]} \Bar{v}_i\,,
\end{align*}
i.e., our test strategy flags all test groups within the cluster with highest utility as benign ($t=0$) and all other test groups as containing malicious clients ($t=1$).

\subsection{Communication Cost}
\label{sec:comm_cost}

\NAME{} introduces a communication overhead only in the round(s) where group testing is performed. 
Considering a single-round of \NAME{}, which is the setting of our experiments, the overall communication cost of the FL consists of: 
\begin{enumerate}
    \item {\bf Before group testing}: Number of communication rounds $\times$ communication complexity of secure aggregation with $n$ clients.
    \item {\bf Group testing round}: $m$  $\times$ communication complexity of secure aggregation with $\max_{j\in [m]} \lvert \mathcal{P}_j \lvert$ clients (maximum size of groups).
    \item {\bf After group testing}: Number of rounds $\times$ communication complexity of secure aggregation with $n-\lvert \hat{\calM} \lvert$ clients (number of clients classified as benign from the group testing).
\end{enumerate}
The overall communication cost heavily depends on the scheme used for secure aggregation. In \cite[Table I]{Jah22}, the communication cost for different secure aggregation schemes is tabulated. 
For example,  LightSecAgg~\cite{So22} has a total communication complexity per round of $\mathcal{O}(n(C+1))$, where $C$ is the model size.
In this case, over $K$ training rounds, \NAME{} yields a total communication complexity of $\mathcal{O}((C+1)(nK + m\max_{j\in [m]} \lvert \mathcal{P}_j \lvert)$ where we have assumed $\hat{\mathcal{M}} = \varnothing$, the worst-case scenario from a communication perspective.

\section{Decoding: inferring the defective vector $\bfd$}
\label{sec:Decoding}

Given the test results $\bft$ and the assignment matrix $\bfA$, 
\NAME{} estimates the defective vector $\bfd$.  In this section, we present two decoding strategies based on probabilistic decision metrics to estimate $\bfd$. 

\subsection{Strategy 1: Neyman-Pearson Based Inference}
\label{subsec:Neyman-Pearson}
In our first strategy, %
we consider optimal inference in a Neyman-Pearson sense, which prescribes for some $\Delta'>1$
\begin{align*}
\hd_i=\left\{\begin{array}{cl}0 & \text{if}\;\; \Pr(\bft|d_i=0)>\Pr(\bft|d_i=1)
\Delta'  \\ 1 & \text{if}\;\; \Pr(\bft|d_i=0)<\Pr(\bft|d_i=1)\Delta'   \end{array}\right.\,.
\end{align*}
The Neyman-Pearson criterion can be rewritten in terms of the log-likelihood ratio (LLR) $L_i=\log(\Pr(\bft|d_i=0)/\Pr(\bft|d_i=1))$ as
\begin{align}
\label{eq:NP}
\hd_i=\left\{\begin{array}{cl}0 & \text{if}\;\; L_i>
\Delta \\ 1 & \text{if}\;\; L_i<\Delta   \end{array}\right.\,,
\end{align}
where $\Delta=\log(\Delta')$.  Further, we can write the LLR $L_i$ as
\begin{align}
\label{eq:LLR}
L_i&=\log\left(\frac{\Pr(d_i=0|\bft)}{\Pr(d_i=1|\bft)}\right)-\log\left(\frac{\Pr(d_i=0)} {\Pr(d_i=1)}\right) \\
&=L_i^{\APP}-\log\left(\frac{1-\delta}{\delta}\right)\,,
\end{align}
where $\delta$ is the \emph{prevalence} of malicious clients in the population of $n$ clients, i.e., the probability of a client being malicious, $\delta=\Pr(d_i=1)$. In a frequentist approach to probability, $\delta=\nm/n$.
Using~\eqref{eq:LLR}, the Neyman-Pearson criterion in \eqref{eq:NP} can be rewritten in terms of the a posteriori LLR $L_i^{\APP}$ as
\begin{align}
\label{eq:NPAPP}
\hd_i=\left\{\begin{array}{cl}0 & \text{if}\;\; L_i^{\APP}>
\Lambda \\ 1 & \text{if}\;\; L_i^{\APP}<\Lambda   \end{array}\right.\,,
\end{align}
where 
\begin{align}
\label{eq:delta}
\Lambda=\Delta+\log\left(\frac{1-\delta}{\delta}\right)\,.
\end{align}

In general, if $\Lambda$ increases, then $P_{\msFA}$ increases and $P_{\msMD}$ decreases. Note that $\Lambda$ depends on the prevalence, i.e., the number of malicious clients $\nm$, {which  is in general not known. In the following, we provide the means to estimate $\nm$.}

\textbf{Estimation of the number of malicious clients.} {For a given $\nm\in [\nmmax] \cup \{0\}$, we consider all patterns of $\nm$ malicious clients and define  $Z$ as the random variable representing the number of zero syndromes, i.e.,  
\begin{equation}\label{eq:syn_sum}
     Z = \sum_{i=1}^m \mathbbm{1}\{S_i = 0\}\,, 
\end{equation} 
where $\mathbbm{1}\{\cdot\}$ is the indicator function. Also, define
\begin{equation}
     z = \sum_{i=1}^m \mathbbm{1}\{s_i = 0\}
\end{equation} 
and 
\begin{equation}
     \hat{z} = \sum_{i=1}^m \mathbbm{1}\{t_i = 0\}\,. 
\end{equation} 
Note that, for a noiseless test, $\hat{z} = z$.}

The decoder has the vector of test results $\bft$ at its disposal. This information can be used to estimate $\nm$ via the maximum likelihood criterion as  
\begin{equation}
\label{eq:z_cond_nm}
    \nmest = \argmax_{\nm}\; \Pr(Z = \hat{z} \lvert N_{\mathsf{m}}=\nm )\,.
\end{equation}
 The likelihood $\Pr(Z = \hat{z} \lvert N_{\mathsf{m}}=\nm )$ can be computed exactly for small enough assignment matrices $\bfA$. Note that the accuracy of the estimate is expected to deteriorate with increasing test noise.

We use the estimate $\nmest$  to estimate the prevalence as $\hat{\delta}=\nmest/n$. The estimated prevalence $\hat{\delta}$ can then be used in \eqref{eq:delta} to obtain $\Lambda$.
However, one must still choose  $\Delta$.
To this end, we consider an ideal setting, i.e., $\bft = \bfs$ and $\nmest=\nm$, and find
\begin{equation}\label{eq:delta_est}
    \hat{\Delta}(\nm) = \argmin_{\Delta} \{ \mathbb{E}[\beta P_{\msMD} + (1-\beta)P_{\msFA}]\}, \nm \in [\nmmax]\,,%
\end{equation}
where $\beta \in[0,1]$ weights between false alarm and misdetection and their dependency on $\Delta$ is implicit.
The expectation is with respect to the $\nm$ malicious clients being sampled uniformly at random and can be computed via Monte-Carlo estimation.
Notably,~\eqref{eq:delta_est} can be solved offline.
During decoding, we set $\Lambda = \hat{\Delta}(\nmest) + \log\left(\frac{1-\Hat{\delta}}{\Hat{\delta}}\right)$.

Hereafter, we refer to FedGT with decoding strategy $1$ as \NAME{}-$\Delta$.
We remark that the number of clients flagged as malicious by FedGT-$\Delta$ may differ from the estimated value $\nmest$, i.e.,  $w_\mathsf{H}(\hat{\bfd})$ is not necessarily equal to $\nmest$.

\subsection{Strategy 2: Flagging $\nmest$ Clients}\label{subsec:nmest}

We propose an alternative strategy to  infer the defective vector $\hat{\bfd}$ by  relying entirely on the estimated number of malicious clients $\nmest$. 
This strategy is based on the observation that the a posteriori LLRs $L_i^{\APP}$ indicate the likelihood of a client being benign (see~\eqref{eq:LLR}), with higher values indicating to a more confident guess. 
Accordingly, Strategy 2 declares   the $\nmest$ clients with smallest $L_i^{\APP}$ as malicious and the remaining clients as benign. 

Let $\bfL^{\APP} = \left(L_1^{\APP}, L_2^{\APP}, \dots, L_n^{\APP}\right)$ be the vector  containing the a posteriori LLRs for all clients and  $\widetilde{\bfL}^{\APP} = \left(L_{i_1}^{\APP}, L_{i_2}^{\APP},\dots, L_{i_n}^{\APP}\right)$ be a sorted version of $\bfL^{\APP}$ with LLRs ordered in ascending order, i.e., $L_{i_j}^{\APP}\ge L_{i_k}^{\APP}$ for $j>k$. 
For an estimated number of malicious clients $\nmest$, we define the decision rule as
\begin{align}
\label{eq:decision_nmest}
\hd_i=\left\{\begin{array}{cl}1 & \text{if}\;\; i \in \left\{i_1, i_2, \dots, i_{\nmest}\right\} \\ 0 & \text{otherwise}\end{array}\right.\,,
\end{align}
where $\{i_1, i_2, \dots, i_{\nmest}\}$ is the set of the indices of the $\nmest$ smallest elements in $\bfL^{\APP}$. 
Note that using this decision strategy, contrary to FedGT-$\Delta$, the number of nonzero entries in $\hat{\bfd}$ is always  $\nmest$, i.e., $w_\mathsf{H}(\hat{\bfd})=\nmest$.
Henceforth, we refer to FedGT with decoding strategy $2$ as \NAME{}-$\nmest$.

Both \NAME{}-$\Delta$ and \NAME{}-$\nmest$, require the a posteriori LLRs $L_i^{\APP}$. 
For not-too-large matrices $\bfA$, they can be computed efficiently  via the forward-backward algorithm  \cite{Bah74},  which exploits  the trellis representation of the assignment matrix $\bfA$. For large matrices $\bfA$, the computation of the a posteriori LLRs is not feasible, and one needs to resort to suboptimal decoding strategies, such as belief propagation \cite{Ksc01}.

Given our focus on the cross-silo setting,  where the number of clients is limited, we  next present how to obtain the a posteriori LLRs for this setting using the trellis representation of the assignment matrix $\bfA$ and the forward-backward algorithm.
In Section~\ref{sec:Trellis}, we describe how to obtain the trellis diagram for a given assignment matrix $\bfA$, and in Section~\ref{sec:BCJR}, we discuss the forward-backward algorithm to compute the a posteriori LLRs to infer $\bfd$.

\subsection{Trellis Representation of Assignment Matrix $\bfA$}
\label{sec:Trellis}

\begin{figure}[t]
    \centering
    \resizebox{!}{.37\linewidth}{
    \begin{tikzpicture}
\tikzstyle{mynode}=[draw,circle,minimum width=1cm,inner sep = 0cm, font = \huge];

\matrix (m) [column sep=25mm, row sep=1cm, ampersand replacement=\&]
{
\node{}; \& [-30mm] 
\node(m01){\huge $\ell=0$}; \& 
\node(m02){\huge $\ell = 1$}; \& 
\node(m03){\huge $\ell = 2$}; \& 
\node(m04){\huge $\ell = 3$}; \& 
\node(m05){\huge $\ell = 4$}; \& 
\node(m06){\huge $\ell = 5$}; \\[-0.5cm]
\node{}; \&
\node[mynode] (m11) {$0$}; \&
\node[mynode] (m12) {$0$}; \&
\node[mynode] (m13) {$0$}; \&
\node[mynode] (m14) {$0$}; \&
\node[mynode] (m15) {$0$}; \&
\node[mynode] (m16) {$0$}; \\
\node{}; \& \&
\node[mynode] (m22) {$1$}; \&
\node[mynode] (m23) {$1$}; \&
\node[mynode] (m24) {$1$}; \& 
\node[mynode] (m25) {$1$}; \&
\node[mynode] (m26) {$1$}; \&
\& \\
\node{}; \& \& \& \&
\node[mynode] (m34) {$2$}; \&
\node[mynode] (m35) {$2$}; \&
\node[mynode] (m36) {$2$}; \\
\node{}; \& \& \&
\node[mynode] (m43) {$3$}; \&
\node[mynode] (m44) {$3$}; \&
\node[mynode] (m45) {$3$}; \&
\node[mynode] (m46) {$3$};  \\
};
\path[draw,dashed] (m11)  --  (m12) --  (m13) --  (m14) --  (m15)-- (m16);
\draw[dashed] (m22) -- (m23);
\draw[dashed] (m23) -- (m24);
\draw[dashed] (m24) to[bend left] (m25);
\draw[dashed] (m25) -- (m26);
\draw[dashed] (m43) to[bend left] (m44);
\draw[dashed] (m44) to[bend left] (m45);
\draw[dashed] (m45) to[bend left] (m46);
\draw[dashed](m34) -- (m35) ;
\draw[dashed] (m35) to[bend left] (m36);
\draw (m35) to[bend right] (m36);
\draw (m15) -- (m36);
\draw (m25) -- (m46);
\draw (m22) -- (m43);
\draw (m23) -- (m44);
\draw (m11) -- (m22);
\draw (m12) -- (m43);
\draw (m13) -- (m34);
\draw (m14) -- (m25);
\draw (m24) to[bend right] (m25);
\draw (m34) -- (m45);
\draw (m43) to[bend right] (m44);
\draw (m44) to[bend right] (m45);
\draw (m45) to[bend right] (m46);
\end{tikzpicture}}
    \caption{Trellis representation of matrix $\bfA$ in Example~\ref{ex:ExampleBipartiteGraep}. The dashed edges correspond to the symbol ``0'', while the solid edges correspond to the symbol ``1''.}
    \label{fig:trellis_A1}
    \vspace{-3ex}
\end{figure}

In this section, we describe the trellis representation corresponding to assignment matrix $\bfA$, which can be used to compute the a posteriori LLRs as described in Section~\ref{sec:BCJR}. The trellis representation was originally introduced for linear block codes in 
\cite{Wol78} and applied to group testing in \cite{Liv21}.

For a given defective vector $\tbfd$ (not necessarily the true one),  define the \emph{syndrome vector} $\tbfs=(\ts_1,\ldots,\ts_n)$, where $\ts_i$  is given by
$\ts_i=\bigvee_{j \in \calP_i} \td_j$.
The syndrome vector can  be written as a function of the defective vector $\tbfd$ and the assignment matrix as $\tbfs=\tbfd\vee \bfA\transpose$.
 Note that several defective vectors are compatible with a given syndrome $\tbfs$. Let $\calD$ be the set of all possible defective vectors, i.e., all binary tuples of length $n$. We denote by $\calD_{\bfs}$ the set of defective vectors  compatible with syndrome vector $\bfs$, i.e., $\calD_{\tbfs}=\{\tbfd\in\calD:\tbfd\vee\bfA\transpose=\tbfs\}$. 

Let $\bfa_j$ be the  $j$-th column of matrix $\bfA$.  %
The syndrome corresponding to defective vector $\tbfd$ can then be rewritten as $\tbfs=\bigvee_{i=1}^n(\td_i \land \bfa_i\transpose)$. This equation naturally leads to a trellis representation of the assignment matrix $\bfA$ as explained next. A trellis is 
a graphical way to represent matrix $\bfA$, consisting of a collection of nodes connected by  edges. %
The trellis  corresponding to   matrix $\bfA$  in Example~\ref{ex:ExampleBipartiteGraep} is depicted in Fig.~\ref{fig:trellis_A1}. Horizontally, the nodes, called trellis states, are grouped into sets indexed by parameter $\ell\in\{0,\ldots,n\}$, referred to as the trellis depth.

Let $\tbfs_\ell$ be the \emph{partial} syndrome vector at trellis depth $\ell\in[n]$ corresponding to $\tbfd$, given as $\tbfs_\ell=\bigvee_{i=1}^\ell(\td_i \land \bfa_i\transpose)$. It is easy to see that $\tbfs_\ell$ can be obtained from $\tbfs_{\ell-1}$ as $\tbfs_\ell=\tbfs_{\ell-1}\vee(\td_\ell \land \bfa_\ell\transpose)$, 
with $\tbfs_0$ being the all-zero vector. The trellis representation is such that each state in the trellis represents a particular partial syndrome. The trellis is then constructed as follows: At trellis depth $\ell=0$ there is a single trellis state corresponding to $\tbfs_0$. At trellis depth $\ell\in[n]$, the trellis states correspond to all possible partial syndrome vectors $\tbfs_\ell$  for all possible partial syndrome vectors $(\td_1,\ldots,\td_\ell)$, with $\td_i\in\{0,1\}$. For example,  
at trellis depth $\ell=1$ there are 
only two trellis states, corresponding to partial syndromes $0\land \bfa_1\transpose=(0,\ldots,0)$ and $1\land \bfa_1\transpose=(a_{1,1},\ldots,a_{1,m})$, i.e., for $\td_1=0$ and $\td_1=1$, respectively. Note that at trellis depth $\ell=n$, there are $2^m$ trellis states, corresponding to all possible syndromes $\tbfs$. For simplicity, we label the trellis state corresponding to partial syndrome vector $\tbfs_\ell=(s_{\ell,1},\ldots,s_{\ell,m})$ by its decimal representation $\sum_{i=1}^m \ts_{\ell,i}2^{i-1}$. Finally, an edge from the node at trellis depth $\ell$ corresponding to partial syndrome $\tbfs_\ell$ to the node at trellis depth $\ell+1$ corresponding to partial syndrome $\tbfs_{\ell+1}$ is drawn if $\tbfs_{\ell+1}=\tbfs_{\ell}\vee(\td_{\ell+1} \land \bfa_{\ell+1}\transpose)$, with $\td_{\ell+1}\in\{0,1\}$. The edge is labeled by the value of $\td_{\ell+1}$ enabling the transition between $\tbfs_\ell$ and $\tbfs_{\ell+1}$.
\begin{example}
For the trellis of Fig.~\ref{fig:trellis_A1}, corresponding to the assignment matrix $\bfA$ in Example~\ref{ex:ExampleBipartiteGraep} with $n=5$ nodes and $m=2$ tests, the number of trellis states at trellis depth $\ell=5$ is $2^2=4$, i.e., all length-$2$ binary vectors (in decimal notation $\{0,1,2,3\}$). At trellis depth $\ell=2$, there are three states, corresponding to all possible partial syndromes $\tbfs=\bigvee_{i=1}^2(\td_i \land \bfa_i\transpose)$, i.e., all possible (binary) linear combinations of the two first columns of matrix $\bfA$, resulting in states $(0,0)\vee(0,0)=(0,0)=0$, $(0,0)\vee(1,1)=(1,1)=3$, $(1,0)\vee(0,0)=(1,0)=1$, and $(1,0)\vee(1,1)=(1,1)=3$. %
\end{example}

The trellis graphically represents all possible defective vectors $\tbfd$ and their connection to the syndromes $\tbfs$ via the assignment matrix $\bfA$. In particular, the paths along the trellis originating in the all-zero state at trellis depth $\ell=0$ and ending in trellis state $\tbfs$ at trellis depth $\ell=n$ correspond to all defective vectors $\tbfd$ compatible with syndrome $\tbfs$.

\subsection{The Forward-Backward Algorithm}
\label{sec:BCJR}

The a posteriori LLRs can be computed efficiently using the trellis representation of matrix $\bfA$ introduced in the previous subsection via the forward-backward algorithm \cite{Bah74}. 
Let $\calE_\ell^{(0)}$ and $\calE_\ell^{(1)}$ be the set of edges connecting trellis states at trellis depth $\ell-1$ with states at trellis depth $\ell$ labeled by  $\td_\ell=0$ and  $\td_\ell=1$, respectively. The a posteriori LLRs $L_\ell^{\APP}$ can be computed as 
\begin{align}
\label{eq:APPLLR}
L_\ell^{\APP}&=\log \sum_{(\sigma',\sigma)\in\calE_\ell^{(0)}}\alpha_{\ell-1}(\sigma')\gamma(\sigma',\sigma)\beta_\ell(\sigma) \nonumber \\
&\qquad
-\log \sum_{(\sigma',\sigma)\in\calE_\ell^{(1)}}\alpha_{\ell-1}(\sigma')\gamma(\sigma',\sigma)\beta_\ell(\sigma)\,,
\end{align}
where $(\sigma',\sigma)$ denotes an edge connecting state $\sigma'$ at trellis depth $\ell-1$ with state $\sigma$
at trellis depth $\ell$.

The quantities $\alpha_{\ell-1}(\sigma')$ and $\beta_\ell(\sigma)$ are called the forward and backward metrics, respectively, and can be computed using the recursions
\begin{align*}
\alpha_{\ell}(\sigma)&=\sum_{\sigma'}\alpha_{\ell-1}(\sigma')\gamma_\ell(\sigma',\sigma)\,, \\
\beta_{\ell-1}(\sigma')&=\sum_{\sigma}\beta_{\ell}(\sigma)\gamma_\ell(\sigma',\sigma)\,,
\end{align*}
with initialization of the forward recursion $\alpha_{0}(0)=1$ and of the backward recursion $\beta_{n}(\sigma)=Q(\bft|\bfs(\sigma))$, where $\bfs(\sigma)$ is the syndrome corresponding to trellis state $\sigma$.
The quantity $\gamma_\ell(\sigma',\sigma)$ is called the branch metric and is given by
\begin{align*}
\gamma_\ell(\sigma',\sigma)=\left\{\begin{array}{cl}1-\delta & \text{if}\;  (\sigma',\sigma)\in\calE_\ell^{(0)} \\ \delta & \text{if}\;  (\sigma',\sigma)\in\calE_\ell^{(1)} \end{array}\right.\,.
\end{align*}

The a posteriori LLRs computed via \eqref{eq:APPLLR} are then used to make decisions on $\{d_i\}$ according to \eqref{eq:NPAPP}. 

\subsection{\NAME{} Hyperparameters} \label{subsec:FedGT_hyper}%

As discussed in the previous section, the decoder requires the distribution $Q(\bft|\bfs)$ and the prevalence $\delta${, which are in general unknown. For the prevalence, we use the estimate $\hat{\delta} = \nicefrac{\nmest}{n}$ as outlined in Section~\ref{subsec:Neyman-Pearson}\label{subsec:Neyman-Pearson}. 
(For the case where the estimated number of malicious clients is zero,  $\nmest = 0$, we do not run the decoder and flag all clients as benign).} %
{On the other hand, the distribution $Q(\bft|\bfs)$, i.e., the noisiness of the test, is test-dependent and hard to estimate. Here, we assume a simple model for $Q(\bft|\bfs)$ which, as shown in the experiments section (Section~\ref{sec:experiments}), yields excellent results. In particular, we assume  that   $Q(\bft|\bfs)$  factorizes as $Q(\bft|\bfs)=\prod_{i=1}^m Q(t_i\lvert s_i)$  and model $Q(t_i\lvert s_i)$ as a binary symmetric channel (BSC), i.e., 
$Q(t_i\lvert s_i)=1-p$ if $t_i=s_i$ and $Q(t_i\lvert s_i)=p$ if $t_i\neq s_i$. In words, we assume that, for each group, the result of the test  is erroneous with probability $p$.}

Our model for $Q(\bft|\bfs)$ requires a single parameter, $p$. Using the correct value of $p$ improves the decoder performance in terms of  misdetection and false-alarm probabilities. However, even with the proposed simple BSC model, accurately estimating $p$ is challenging. Therefore, we arbitrarily select a value for $p$ and demonstrate that our decoder remains robust to this choice (see Section~\ref{sec:rob_p}). Specifically, we choose a small value for $p$ (as a relatively accurate test is preferred), namely $p=0.05$.

FedGT-$\Delta$ also requires choosing parameter  $\beta$ (see in~\eqref{eq:delta_est}), which balances the misdetection and false-alarm probabilities. The impact of a higher false alarm or misdetection probability depends on the scenario. For instance, when facing a powerful attack, a near-zero misdetection probability is preferable. Conversely, in heterogeneous settings, a low false alarm is crucial to avoid penalizing correct and unique data points. If prior knowledge of the scenario is available, one can set $\beta <0.5$ to empashize lowering the false-alarm probability or $\beta > 0.5$ to prioritize reducing the misdetection probability. Here, we assume no prior knowledge and set $\beta=0.5$, meaning we weight  misdetections and false alarms equally. 

Overall, since we fix $p$ and (for FedGT-$\Delta$) $\beta$ independently of the dataset and the nature of the attack, FedGT requires no hyperparameter tuning.

\section{Experiments}
\label{sec:experiments}

\subsection{Setup}
\label{subsec:Hyperparameters}
We consider a cross-silo scenario with $n=15$ clients (all  participating in each training round) out of which $\nm$ are malicious. 
In Section~\ref{sec:more_devices}, we also provide results for $n=30$ clients. 
We remark that these numbers are aligned with current cross-silo applications~\cite{Ogier2022flamby, heyndrickx2023melloddy, Shej21}.
The goal of the server is to prevent an attack by identifying the malicious clients and exclude their models from the global aggregation.
The experiments are conducted for image classification problems on the MNIST~\cite{MNIST}, CIFAR-10~\cite{He16}, and ISIC2019~\cite{Codella} datasets for which we rely on a single-layer neural network, a ResNet-18~\cite{Kri09}, and an Efficientnet-B0 pretrained on Imagenet~\cite{tan2019efficientnet}, respectively.

Similar to previous works~\cite{Mallah21,nguyen22,Pan20, Cao20, Park21}, we assume that the server has a small validation dataset at its disposal to perform the group tests (the validation dataset is not used for training).
Such dataset is not required by FedGT, but is used here due to our choice for the tests in the experiments.
The validation dataset should contain data that are sampled from a distribution close to the underlying distribution of the (benign) clients' datasets, i.e., it should be a \emph{quasi-dataset}~\cite{Pan20, Mallah21}.
For the experiments, we create the validation dataset by randomly sampling $100$ data-points from the available data. As a result, the label distribution may not be uniform. 
For MNIST and CIFAR10, the remaining data points (of size $59900$ and $49900$) are split evenly at random among the 15 clients, resulting in homogenous data among the clients, and used for training. 
For ISIC2019, we follow~\cite{Ogier2022flamby} and randomly partition the dataset into a training and a test set consisting of $19859$ and $3388$ samples, respectively.
We then partition the training dataset into six parts according to the image acquisition system used to collect the images.
Finally, we iteratively split the largest partition in half until we have $15$ partitions.
This procedure results in a heterogeneous setting where both label distributions and number of samples differ among clients (for the details of the ISIC2019 experiments we refer the reader to the Appendix).

For MNIST, we use the cross-entropy loss and stochastic gradient descent with a learning rate of $0.01$, batch size of $64$, and number of local epochs equal to $1$.
For CIFAR-10, we use the cross-entropy loss and stochastic gradient descent with momentum and parameters taken from~\cite{McM17}: the learning rate is $0.05$, momentum is $0.9$, and the weight decay is $0.001$.
Furthermore, the batch size is set to $128$ and the number of local epochs is set to $5$.
For ISIC2019, we use the focal loss in~\cite{Lin2017focal} and stochastic gradient descent with a learning rate of $0.0005$, momentum of $0.9$, and weight decay equal to $0.0001$.
The batch size equals $64$ and the number of local epochs is set to $1$.
Furthermore, we use the same set of augmentations as in~\cite{Ogier2022flamby} to encourage generalization during the training.
The results presented are averaged over $10$, $5$, and $3$ runs for MNIST, CIFAR-10, and ISIC2019, respectively. 

For the experiments over ISIC2019, due to the heterogeneous client data (see Appendix), the identities of the malicious clients, i.e., the realizations of vector $\bfd$, significantly impact the results. Therefore, for $\nm >0$, we run the experiments $3$ times with different realizations of $\bfd$ but the same client data distribution. In particular, we evaluate three different scenarios: i) the very heterogeneous clients (clients $4$ and $10$) are not malicious; ii) only one of them is malicious, and iii) both of them are malicious.%

{In our experiments, we use the test strategy outlined in Section~\ref{sec:TestDesign} within FedGT. In particular,} for the experiments over MNIST and CIFAR-10, we use $\silhouettethres = 0.6$ (see \eqref{eq:hatk}) and due to the heterogeneity of ISIC2019, we use $\silhouettethres = 0$, i.e., the clustering solution is decided solely from the Dunn index.

We show the performance of FedGT using both FedGT-$\nmest$ and FedGT-$\Delta$.
We compare their performance to three benchmarks: ``no defense'', ``oracle'' and RFA~\cite{Pillutla22}. 
The no defense benchmark corresponds to plain {\FL} including all clients, i.e., disregarding some clients may be malicious, while
 the oracle is an ideal setting where the server knows the malicious clients and discards them.   
Note that RFA belongs to a short list of defense mechanisms that also provide privacy.%

To demonstrate the effectiveness of FedGT, we perform the group testing step only once during the training.
This constitutes the weakest version of our framework as the group testing may be performed in each round at the expense of increased communication cost (see Section~\ref{sec:comm_cost}).
In particular, for MNIST, we perform the group testing in the first round and for CIFAR-10 and ISIC2019, in the fifth round.
We pick as the assignment matrix a parity-check matrix of a BCH code~\cite{Bos60} of length $15$ and redundancy $8$, meaning that we create a group testing scheme where the $15$ clients are pooled into $8$ groups, each containing $4$ clients. This choice of $\bfA$ allows for $\nmmax = 5$, where $\kappa$ in~\eqref{eq:z_cond_nm} is set to $20\%$. 
Also, as discussed in Section~\ref{subsec:FedGT_hyper} we set $\beta=0.5$ in~\eqref{eq:delta_est} for FedGT-$\Delta$ (tuning $\beta$ may yield better performance, especially for the experiments over ISIC2019, but requires prior knowledge). 

For the considered setup, \NAME{} yields a communication overhead and privacy guarantee as follows.
\begin{itemize}
    \item \textbf{Communication cost.} Considering a secure aggregation scheme with linear communication complexity such as LightSecAgg~\cite{So22}, the communication cost of the group testing round is approximately $2\times$ the complexity of secure aggregation with 15 clients. 
    Compared to RFA, which requires $3\times$ communication cost of secure aggregation with 15 clients \textit{in each round}, \NAME{} yields a significantly reduced communication cost. 

\item \textbf{Privacy.} With our choice of assignment matrix,  {\NAME} guarantees the same level of privacy of full secure aggregation with $4$ clients. This stems from the property that any linear combination of the server's group aggregates leads to an aggregation involving no fewer than $4$ client models, as elucidated in Proposition~\ref{prop:privacy}.
\end{itemize}

\subsection{Robustness Toward the Crossover Probability $p$} \label{sec:rob_p}

The decoding strategy employed in FedGT-$\Delta$ requires selecting the optimal parameter $\Delta$, which, due to the trellis-based decoding approach, depends on the unknown crossover probability $p$. In our experiments, we set $p=0.05$. Next, we empirically demonstrate that \NAME{}-$\Delta$ is robust to a mismatch in the assumed $p$.

In Table~\ref{tab:robustness_p}, we present the value of the objective $0.5 P_{\msMD} + 0.5 P_{\msFA}$ for different values of $p$ and $\nm$ when $\Delta$ is obtained via~\eqref{eq:delta_est} and $p=0.05$ is believed to be the true value, i.e., we assess the impact of a mismatch in $p$.
It can be seen that the objective function is robust to a mismatch in $p$ for all $\nm\in[\nmmax]$.
Hence, $p$ may be chosen to hedge for the anticipated noise in the testing strategy and one does not have to be concerned about the impact of a mismatch on the choice of $\Delta$. 
\begin{table}[t]
\caption{Robustness of the objective, $0.5 P_{\msMD} + 0.5 P_{\msFA}$, with varying $p$ for a $\Delta$ obtained from~\eqref{eq:delta_est} with $p=5\%$.}
\vspace{-3ex}
\begin{center}
\resizebox{\linewidth}{!}{%
\begin{tabular}{cccccccccc}
\toprule
\diagbox{$\nm$}{$p$} & $1\%$ & $2.5\%$ & $5\%$ & $7.5\%$ & $10\%$ & $12.5\%$ & $15\%$ & $17.5\%$ & $20\%$ \\
\midrule
$1$ & 0  & 0 & 0  & 0& 0  & 0& 0  & 0 & 0\\
$2$ & 0.01 & 0.01 & 0.01 & 0.01 & 0.01 & 0.01 & 0.01 & 0.02 & 0.02 \\
$3$ & 0.07 & 0.07 & 0.07 & 0.07 & 0.07 & 0.07 & 0.06 & 0.06 & 0.06 \\
$4$ & 0.14 & 0.14 & 0.14 & 0.14 & 0.14 & 0.14 & 0.15 & 0.15 & 0.16 \\
$5$ & 0.15 &0.15 & 0.15 & 0.15 & 0.15 & 0.15 & 0.15 & 0.15 & 0.17 \\
\bottomrule
\end{tabular}
}
\end{center}
\label{tab:robustness_p}
\vspace{-4ex}
\end{table}

We observed that  FedGT-$\nmest$ demonstrates an even greater robustness to a mismatch in $p$. However, due to space constraints, we omit a similar table for FedGT-$\nmest$.

\subsection{Experimental Results for Targeted Attacks}
\label{subsec:targeted}

For targeted data-poisoning, we consider label-flipping attacks.
We refer to the attacked label as the source label and the resulting label after the flip as the target label.
For MNIST, we consider malicious clients to flip  source label $1$ into  target label $7$.
As such, the objective of the malicious clients is to cause the global model to misclassify $1$'s into $7$'s.
Similarly, for CIFAR-10, malicious clients change  source label $7$, i.e., horses, into  target label $4$, i.e., deers. %
For the ISIC2019 dataset,  malicious clients mislabel  source label 0, i.e., melanoma, into  target label 1, i.e., mole. Note that this attack has a significant medical impact, as the goal of the attacker is to force the model to classify cancer into non-cancer.
Since the adversary's goal  is not to deteriorate the global model but to make it misinterpret the source label as the target label, we adopt the \emph{attack accuracy} as the primary metric of interest. The attack accuracy is defined as the fraction of source labels  classified as the target label in the test dataset.
Moreover, since a successful defense mechanism should not compromise the overall utility of the model, we employ the accuracy on the test dataset as a secondary performance metric.
\begin{table*}[t]
\caption{Attack accuracy (ATT) and top-1 accuracy (ACC) measured after specified communication rounds for MNIST, CIFAR10, and ISIC2019 datasets. All entries are provided as mean and standard deviation with values in $\%$.}
\vspace{-3ex}
\begin{center}
\resizebox{\linewidth}{!}{%
\begin{tabular}{ccc|cc|cc|cc|cc}
\toprule
& \multicolumn{2}{c}{Oracle} &\multicolumn{2}{c}{RFA~\cite{Pillutla22}} &\multicolumn{2}{c}{\NAME{}-$\nmest$} &\multicolumn{2}{c}{\NAME{}-$\Delta$} &\multicolumn{2}{c}{No defense} \\ \cmidrule(lr){2-3} \cmidrule(lr){4-5} \cmidrule(lr){6-7} \cmidrule(lr){8-9} \cmidrule(lr){10-11}
$\nm$ & ATT $\downarrow$ & ACC $\uparrow$ & ATT $\downarrow$ & ACC $\uparrow$ & ATT $\downarrow$ & ACC $\uparrow$ & ATT $\downarrow$ & ACC $\uparrow$ & ATT $\downarrow$ & ACC $\uparrow$ \\
\midrule
\multicolumn{11}{c}{MNIST (10 communication rounds)} \\
\midrule
$0$ & $0.07 \pm 0.07$ & $90.32 \pm 0.09$        
& $0.07 \pm 0.07$ & $90.32 \pm 0.10$ & $0.08 \pm 0.06 $ & $90.18 \pm 0.10$ & $0.08 \pm 0.06$ & $90.18 \pm 0.11$ & $0.07 \pm 0.07$ & $90.32 \pm 0.09$ \\
$1$ & $0.04 \pm 0.04$ & $90.32 \pm 0.17$        & $0.10 \pm 0.07$ & $90.32 \pm 0.10$ & $0.05 \pm 0.06$ & $90.21 \pm 0.09$ & $0.05 \pm 0.06$ & $90.19 \pm 0.10$ & $0.15 \pm 0.04$ & $90.30 \pm 0.17$ \\
$2$ & $0.04 \pm 0.04$ & $90.33 \pm 0.19$        & $0.13 \pm 0.06$ & $90.31 \pm 0.12$ & $0.06 \pm 0.07$ & $90.17 \pm 0.13$ & $0.07 \pm 0.07$ & $90.16 \pm 0.12$ & $0.18 \pm 0.04$ & $90.23 \pm 0.15$ \\
$3$ & $0.04 \pm 0.04$ & $90.31 \pm 0.17$ & $0.15 \pm 0.04$ & $90.29 \pm 0.10$ & $0.06 \pm 0.07$ & $90.12 \pm 0.09$ & $0.07 \pm 0.07$ & $90.13 \pm 0.08$ & $0.36 \pm 0.13$ & $90.10 \pm 0.16$ \\
$4$ & $0.04 \pm 0.04$ & $90.32 \pm 0.16$ & $0.16 \pm 0.04$ & $90.29 \pm 0.09$ & $0.19 \pm 0.13$ & $90.05 \pm 0.13$ & $0.07 \pm 0.05$ & $90.11 \pm 0.11$ & $1.03 \pm 0.23$ & $89.89 \pm 0.19$ \\
$5$ & $0.04 \pm 0.04$ & $90.34 \pm 0.16$ & $0.17 \pm 0.03$ & $90.26 \pm 0.10$ & $1.53 \pm 2.75$ & $89.77 \pm 0.45$ & $0.07 \pm 0.05$ & $90.07 \pm 0.10$ & $3.15 \pm 0.40$ & $89.49 \pm 0.18$ \\
\midrule
\multicolumn{11}{c}{CIFAR10 (30 communication rounds)} \\
\midrule
$0$ & $4.10 \pm 0.27$ & $81.66 \pm 0.16$ & $3.86 \pm 0.27$ & $81.94 \pm 0.28$ & $4.12 \pm 0.28$ & $81.49 \pm 0.40$ & $4.28 \pm 0.49$ & $81.23 \pm 0.91$ & $4.10 \pm 0.27$ & $81.66 \pm 0.16$ \\
$1$ & $3.36 \pm 0.56$ & $81.69 \pm 0.22$ & $5.44 \pm 0.67$ & $81.65 \pm 0.20$ & $4.84 \pm 1.87$ & $81.07 \pm 0.73$ & $4.40 \pm 1.35$ & $80.41 \pm 2.22$ & $5.72 \pm 0.68$ & $81.45 \pm 0.06$ \\
$2$ & $4.10 \pm 0.83$ & $81.44 \pm 0.22$ & $7.74 \pm 1.84$ & $81.49 \pm 0.39$ & $4.82 \pm 2.55$ & $80.83 \pm 0.41$ & $4.54 \pm 1.75$ & $78.46 \pm 1.97$ & $9.62 \pm 1.72$ & $81.11 \pm 0.30$ \\
$3$ & $3.56 \pm 0.32$ & $81.13 \pm 0.32$ & $11.06 \pm 0.62$ & $81.03 \pm 0.21$ & $4.32 \pm 2.31$ & $80.82 \pm 0.43$ & $4.92 \pm 1.23$ & $79.01 \pm 2.22$ & $17.62 \pm 2.23$ & $80.12 \pm 0.55$ \\
$4$ & $3.94 \pm 1.07$ & $81.07 \pm 0.18$ & $16.92 \pm 3.07$ & $80.52 \pm 0.56$ & $10.7 \pm 5.53$ & $80.28 \pm 0.70$ & $4.90 \pm 1.09$ & $78.68 \pm 1.61$ & $26.42 \pm 2.18$ & $79.25 \pm 0.31$ \\
$5$ & $3.74 \pm 0.43$ & $80.54 \pm 0.11$ & $25.16 \pm 3.94$ & $79.67 \pm 0.37$ & $18.62 \pm 7.87$ & $79.54 \pm 0.76$ & $5.32 \pm 1.19$ & $76.53 \pm 0.47$ & $38.40 \pm 6.48$ & $78.12 \pm 0.64$ \\
\midrule
\multicolumn{11}{c}{ISIC2019 (40 communication rounds)} \\
\midrule
$0$ & $25.04 $ & $63.79$ & $21.72$ & $64.96$ & $16.09$ & $62.91$ & $15.92$ & $60.70$ & $25.87$ & $63.29$ \\
$1$ & $21.72 \pm 1.76$ & $63.14 \pm 0.24$ & $23.27 \pm 0.83$ & $63.24 \pm 1.61$ & $16.97 \pm 0.16$ & $63.28 \pm 0.73$ & $17.69 \pm 0.75$ & $62.17 \pm 1.45$ & $25.43 \pm 2.78$ & $62.00 \pm 0.61$ \\
$2$ & $21.23 \pm 2.51$ & $63.19 \pm 0.81$ & $24.71 \pm 1.86$ & $62.63 \pm 1.37$ & $18.79 \pm 1.96$ & $63.12 \pm 0.38$ & $19.07 \pm 1.64$ & $62.47 \pm 0.97$ & $26.70 \pm 3.16$ & $62.52 \pm 0.51$ \\
$3$ & $21.28 \pm 2.54$ & $62.16 \pm 1.48$ & $29.30 \pm 3.05$ & $62.63 \pm 0.21$ & $18.74 \pm 2.84$ & $62.89 \pm 0.75$ & $19.13 \pm 2.52$ & $58.47 \pm 4.05$ & $30.51 \pm 4.62$ & $61.92 \pm 0.68$ \\
$4$ & $20.18 \pm 2.13$ & $61.65 \pm 0.53$ & $31.29 \pm 2.25$ & $62.10 \pm 1.10$ & $18.57 \pm 1.88$ & $62.58 \pm 0.18$ & $19.90 \pm 3.18$ & $56.15 \pm 4.32$ & $34.11 \pm 2.55$ & $61.50 \pm 1.02$ \\
$5$ & $20.01 \pm 1.85$ & $61.01 \pm 0.40$ & $38.47 \pm 3.52$ & $61.73 \pm 1.15$ & $22.66 \pm 0.55$ & $61.41 \pm 0.19$ & $17.69 \pm 0.86$ & $54.90 \pm 2.85$ & $38.70 \pm 3.57$ & $60.30 \pm 1.40$ \\
\bottomrule
\end{tabular}
}
\end{center}
\label{tab:targeted}
\vspace{-4ex}
\end{table*}

For the utility metric adopted in the testing strategy (see Section~\ref{sec:TestDesign}), we consider the source label recall, i.e., the fraction of source labels  classified into the correct label,  to flag test groups containing malicious clients and perform PCA on the weights of the fully connected layer incoming to the source label. 
We remark that to identify a label under attack, one may simply monitor, e.g., the recall of each label in the different test groups, using the validation dataset. %
For a targeted attack, the test noisiness obtained from our experiments is $5.41\%$, $9.16\%$, and $19.53\%$ for MNIST, CIFAR-10, and ISIC2019, respectively (we recall that we used $p=0.05$ in all our experiments). At first glance, the test results for  ISIC2019 appear to be very noisy. However,  we note that, due to the high heterogeneity, some benign clients may actually harm the model due to their data distribution, even without containing  poisoned data. FedGT identifies some of these clients as malicious\textemdash thus yielding higher utility\textemdash, which explains the higher noisiness of the test results.

In Table~\ref{tab:targeted}, we give the attack accuracy and top-1 (or balanced) accuracy of \NAME{}-$\nmest$ and \NAME{}-$\Delta$. For comparison, we also provide results for  no defense,  oracle, and RFA~\cite{Pillutla22}. The results are shown as the mean and standard deviation in $\%$ obtained from a Monte-Carlo-based simulation approach. However, please note that the experiment over ISIC2019 for $\nm =0$ is performed only once, as we do not investigate  client data distribution other than the one depicted in the Appendix. %

For MNIST, we observe a modest impact of the label flip, even for $\nm =5$. Nevertheless, \NAME{}-$\Delta$ effectively mitigates the attack accuracy compared to no defense. Notably, it significantly outperforms RFA (which lacks the capability of identifying malicious clients and entails a much larger communication complexity) and performs close to the oracle. \NAME{}-$\nmest$ outperforms RFA for $1\le \nm\le 3$, but falls short for other values of $\nm$. 

For CIFAR10, the label flip attack has a significant impact, as can be seen from the no-defense attack accuracy, nearing $40$\% for $\nm=5$. 
Both versions of {\NAME} significantly outperform RFA in terms of attack accuracy for all $\nm\ge 1$, especially for larger values of $\nm$, with \NAME{}-$\Delta$ performing very close to the oracle. For example, for $\nm=5$, \NAME{}-$\nm$ and \NAME{}-$\Delta$  reduce the attack accuracy to $18.32\%$ and $5.32\%$, respectively, compared to $25.16\%$ RFA.
(We note that the pronounced reduction in attack accuracy by \NAME{}-$\Delta$ is achieved at the expense of a slight penalty in  accuracy for larger values of $\nm$).

For ISIC2019, RFA performs poorly, achieving only a small improvement in attack accuracy with respect to no defense (RFA is known to underperform for heterogeneous data across clients \cite{Li23}). 
Both versions of \NAME{} significantly diminish the attack accuracy, even outperform the oracle.
This can be explained from the data heterogeneity across clients where some clients, although not malicious, will be biased to output a given label, see, e.g., client $4$ and client $10$ in the Appendix.
Hence, due to the testing strategy, \NAME{} may identify benign clients exhibiting extreme heterogeneity as malicious to be removed from the training, ultimately reducing the attack accuracy and benefiting the overall utility of the global model.
Compared to the experiments on MNIST and CIFAR10, \NAME{}-$\nmest$ yields the strongest performance over the two metrics.
This is again attributed to heterogeneity, as \NAME{}-$\Delta$ removes too many clients, resulting in a high false-alarm probability. 

In Fig.~\ref{fig:att_acc}, we plot the attack accuracy of the label-flip attack over communication rounds for different values of $\nm$.
From the CIFAR10 and ISIC2019 experiments, the impact of the group testing is clearly seen with the attack accuracy rapidly dropping in round 5.

\subsection{Experimental Results for Untargeted Attacks}
\label{sec:untargeted}

Next, we consider a label permutation attack where malicious clients offset their data labels by $1$, i.e., $L_{\mathrm{new}}=(L_{\mathrm{old}}+1)\mod n_{\mathrm{c}}$, where $n_{\mathrm{c}}$ is the number of classes. 
The attack aims at deteriorating the classification accuracy over all labels, i.e., an attacker wants to lower the top-1 accuracy (MNIST and CIFAR-10) or the balanced accuracy (ISIC2019). 
For this reason, we use the top-1 accuracy (MNIST and CIFAR10) and the balanced accuracy (ISIC2019) on the test groups' aggregates as the qualitative metric in the testing strategy (see Section~\ref{subsec:Hyperparameters}) and perform PCA on the flattened weights of the entire fully connected layer. The balanced accuracy is the average recall per class, used to take into account class imbalances, as in the case of ISIC2019~\cite{Ogier2022flamby}. The test error probability is $2.29\%$, $4.58\%$, and $3.91\%$ for experiments over MNIST, CIFAR-10 and ISIC2019, respectively.
\begin{figure}[t!]
\centering
\resizebox{\linewidth}{!}{
    \input{./plots/targeted_acc_vs_comm.tex}}
    \vspace{-3ex}
    \caption{Average attack accuracy on the MNIST (row 1), CIFAR10 (row 2) and ISIC2019 (row 3) datasets for varying number of malicious clients. These results are obtained from FL experiments where $\nm$ clients out of $n =15$ total clients act as malicious by deploying a label-flip attack.}
    \label{fig:att_acc}
    \vspace{-3ex}
\end{figure}

In Table~\ref{tab:untargeted}, %
we show the top-1 accuracy versus $\nm$ for MNIST and CIFAR-10, and the balanced accuracy for ISIC2019.  The results are tabulated as the mean and standard deviation in $\%$.
For all cases, with no defense, a significant drop in accuracy is observed as the number of malicious clients grows.  
For MNIST, \NAME{}-$\Delta$ achieves similar performance to RFA and  oracle for all considered $\nm$. On the other hand, \NAME{}-$\nmest$ performs comparably to RFA for $\nm \leq 3$ but its performance declines for $\nm=4$ and $\nm=5$. For  CIFAR-10,  both versions of \NAME{} perform similar to RFA  for $\nm \leq 4$, but worse  for $\nm=5$. The robust performance of RFA is anticipated due to the untargeted attack rendering malicious client models significantly different from benign ones given the iid data distribution  across clients. Consequently, the geometric median\textemdash essentially performing a majority vote\textemdash assigns the malicious models a very low weight.

For ISIC2019, \NAME{}-$\nmest$ performs better than RFA for all values of $\nm$, except $\nm=1$. Moreover, for $\nm =0$ and $\nm = 2$, \NAME{}-$\nmest$ performs even better than oracle due to the heterogeneity of the data distribution among clients. This means that some clients can be flagged as malicious just because they deteriorate the utility of the global model due to their data samples. 
 \NAME{}-$\Delta$ performs better or similar to RFA for $\nm\le 2$ but worse than  RFA for $\nm \geq 3$. We note that this result is due to some realizations of defective vectors triggering the decoder to falsely flag as malicious clients with more homogeneous data distribution and resort to learning with clients with heterogeneous data.  %

In Fig.~\ref{fig:untarg_att_acc_comm}, we plot the top-1 accuracy for MNIST and CIFAR-10 and the balanced accuracy for ISIC2019 over different communication rounds. For $\nm = 1$, the attack is not very powerful (regardless of the dataset), and the no defense and oracle benchmarks have similar performance. For $\nm \in \{3, 5\}$,  the impact on the top-1 accuracy of the attack for MNIST and CIFAR-10 is significant, as shown by the significant gap between the no defense and oracle curves. {\NAME}-$\Delta$ closes this gap, but RFA outperforms our strategy, due to its majority decision-based aggregation technique. 
For $\nm=5$, \NAME{}-$\nmest$ suffers compared to the other techniques for $\nm = 5$.
This is due to its reliance on an accurate estimate $\nmest$, something that becomes harder with a larger $\nm$ as more test groups are contaminated.
For the  ISIC2019 dataset, \NAME{}-$\nmest$ outperforms RFA for $\nm = 3,5$, while \NAME{}-$\Delta$ performs poorly for $\nm \geq 3$. This occurs due to the heterogeneity of ISIC2019 where a false-alarm incurs a significant penalty on the global model.
\begin{table}[t]
\caption{Top-1 or balanced accuracy (ACC) measured after specified communication rounds for MNIST, CIFAR10, and ISIC2019 datasets, for experiments with untargeted attacks. All entries are provided as mean and standard deviation with values in $\%$.}
\vspace{-3ex}
\begin{center}
\resizebox{\linewidth}{!}{%
\begin{tabular}{cccccc}
\toprule
& \multicolumn{1}{c}{Oracle} &\multicolumn{1}{c}{RFA~\cite{Pillutla22}} &\multicolumn{1}{c}{\NAME{}-$\nmest$} &\multicolumn{1}{c}{\NAME{}-$\Delta$} &\multicolumn{1}{c}{No defense}\\ \cmidrule(lr){2-2} \cmidrule(lr){3-3} \cmidrule(lr){4-4} \cmidrule(lr){5-5} \cmidrule(lr){6-6}
$\nm$ & ACC $\uparrow$ & ACC $\uparrow$ & ACC $\uparrow$ & ACC $\uparrow$ & ACC $\uparrow$ \\
\midrule
\multicolumn{6}{c}{MNIST (10 communication rounds)} \\
\midrule
$0$ & $90.18 \pm 0.10$ & $90.18 \pm 0.09$ & $90.18 \pm 0.10$ & $90.18 \pm 0.10$ & $90.18 \pm 0.10$ \\
$1$ & $90.19 \pm 0.10$ & $90.22 \pm 0.11$ & $90.21 \pm 0.09$ & $90.19 \pm 0.10$ & $89.96 \pm 0.08$ \\
$2$ & $90.18 \pm 0.12$ & $90.20 \pm 0.11$ & $90.14 \pm 0.19$ & $90.10 \pm 0.18$ & $89.01 \pm 0.11$ \\
$3$ & $90.18 \pm 0.09$ & $90.17 \pm 0.09$ & $89.94 \pm 0.28$ & $90.08 \pm 0.16$ & $87.52 \pm 0.13$ \\
$4$ & $90.16 \pm 0.10$ & $90.17 \pm 0.09$ & $88.72 \pm 1.20$ & $89.75 \pm 1.06$ & $85.06 \pm 0.21$ \\
$5$ & $90.17 \pm 0.12$ & $90.16 \pm 0.09$ & $84.96 \pm 5.20$ & $89.61 \pm 1.34$ & $80.36 \pm 0.26$ \\
\midrule
\multicolumn{6}{c}{CIFAR10 (30 communication rounds)} \\
\midrule
$0$ & $81.66 \pm 0.16$ & $81.94 \pm 0.28$ & $81.40 \pm 0.48$ & $80.47 \pm 2.32$ & $81.66 \pm 0.16$ \\
$1$ & $81.94 \pm 0.27$ & $81.64 \pm 0.19$ & $81.64 \pm 0.34$ & $81.67 \pm 0.36$ & $81.49 \pm 0.24$ \\
$2$ & $81.60 \pm 0.15$ & $81.40 \pm 0.21$ & $81.12 \pm 0.45$ & $80.77 \pm 0.71$ & $80.73 \pm 0.08$ \\
$3$ & $81.28 \pm 0.17$ & $81.13 \pm 0.31$ & $80.90 \pm 0.40$ & $79.78 \pm 1.86$ & $77.73 \pm 2.91$ \\
$4$ & $80.99 \pm 0.30$ & $79.78 \pm 0.47$ & $80.40 \pm 0.46$ & $78.92 \pm 1.75$ & $56.49 \pm 1.82$ \\
$5$ & $80.94 \pm 0.39$ & $78.52 \pm 2.43$ & $71.71 \pm 8.70$ & $76.79 \pm 0.76$ & $49.07 \pm 19.31$ \\
\midrule
\multicolumn{6}{c}{ISIC2019 (40 communication rounds)} \\
\midrule
$0$ & $61.88$ & $61.26$ & $62.70$ & $62.74$ & $61.88$ \\
$1$ & $61.63 \pm 1.03$ & $62.07 \pm 0.60$ & $61.80 \pm 0.45$ & $63.13 \pm 0.78$ & $61.03 \pm 1.17$ \\
$2$ & $62.53 \pm 1.40$ & $62.86 \pm 0.66$ & $63.58 \pm 0.19$ & $62.48 \pm 1.66$ & $59.13 \pm 1.06$ \\
$3$ & $61.84 \pm 1.80$ & $60.15 \pm 1.55$ & $61.48 \pm 1.46$ & $57.01 \pm 4.03$ & $54.02 \pm 1.51$ \\
$4$ & $61.34 \pm 0.80$ & $58.87 \pm 1.17$ & $58.88 \pm 3.27$ & $53.27 \pm 2.85$ & $49.75 \pm 0.80$ \\
$5$ & $58.87 \pm 0.19$ & $52.34 \pm 0.64$ & $55.47 \pm 0.71$ & $50.12 \pm 1.89$ & $42.64 \pm 1.96$ \\
\bottomrule
\end{tabular}
}
\end{center}
\label{tab:untargeted}
\vspace{-4ex}
\end{table}

Finally, we observe an interesting phenomenon for the experiments over the CIFAR-10 dataset. For $\nm=5$, the no defense curve exhibits significant fluctuations throughout the rounds.
Although the performance of {\NAME}-$\nmest$ also fluctuates, it does so to a significantly lesser extent, while the fluctuations are more pronounced in \NAME{}-$\Delta$ and RFA. %

\subsection{Federated Learning With More Clients}
\label{sec:more_devices}

Hitherto, the experiments have focused on a cross-silo FL scenario with $15$ clients. Next, we investigate the performance of {\NAME} for a cross-silo FL scenario with a larger number of clients, specifically $n=30$ clients. 
For this scenario, we choose as the assignment matrix the parity-check matrix of a $(30,18)$ cyclic code of length $30$ and dimension $18$, resulting in $12$ groups, each containing $6$ clients. The dual of this cyclic code has minimum Hamming distance $6$, thus {\NAME} preserves the same clients' privacy of secure aggregation with $6$ clients. This choice of $\bfA$ allows for $\nmmax = 8$, where the probability in~\eqref{eq:prop_hist} is constrained to $20\%$, i.e., $\kappa=0.2$.
\begin{figure}[t!]
\centering
\resizebox{\linewidth}{!}{
\input{./plots/untargeted_acc_vs_comm.tex}}
 \vspace{-3ex}
    \caption{Average top-1 accuracy on the MNIST (row 1), CIFAR10 (row 2) and ISIC2019 (row 3) datasets for varying $\nm$.}
    \label{fig:untarg_att_acc_comm}
    \vspace{-3ex}
\end{figure}

We investigate a scenario with $\nm = 6$  malicious clients and both a targeted attack and an untargeted attack. 
We conduct experiments over the MNIST and CIFAR-10 datasets, with the same hyperparameters as specified in Section~\ref{subsec:Hyperparameters}. Due to the relatively high number of clients, we do not run the experiments over ISIC2019, as this dataset is tailored to FL scenarios with a smaller number of clients.

In Fig.~\ref{fig:30clients}, we plot the attack accuracy of the targeted attack (row 1) and the top-1 accuracy of the untargeted attack (row 2), respectively. For a targeted attack, {\NAME-$\Delta$} performs very close to the oracle and outperforms RFA for both datasets, with the improvement in performance  being significant for the CIFAR-10 dataset. For an an untargeted attack, FedGT-$\Delta$ performs similar to RFA and oracle.

\section{Conclusion}\label{sec:conclusion}

We proposed FedGT, a novel and flexible framework for identifying malicious clients in {\FL} that is compatible with secure aggregation and does not require  hyperparameter tuning. 
By grouping clients into overlapping groups, FedGT 
enables the 
identification of malicious clients at the expense of secure aggregation involving fewer clients. 
Experiments conducted in a cross-silo scenario for different data-poisoning attacks demonstrate the effectiveness of FedGT in identifying malicious clients, resulting in high model utility and low attack accuracy. Remarkably, FedGT significantly outperforms the  recently-proposed robust federated aggregation (RFA) protocol based on the geometric median (which is unable to identify malicious clients and entails a much higher communication cost) across multiple scenarios. 
To the best of our knowledge, this is the first work that provides a solution for identifying malicious clients in {\FL} with secure aggregation.

\begin{figure}[t!]
    \centering
    \resizebox{\linewidth}{!}{\definecolor{chocolate}{rgb}{0.48, 0.25, 0.0}
\definecolor{amber}{rgb}{1.0, 0.75, 0.0}
\begin{tikzpicture}
[spy using outlines={rectangle, magnification=6, size=4cm, connect spies}]

\begin{groupplot}[ 
    group style={
        group size=2 by 2,
        horizontal sep=2cm, vertical sep =.7cm,
   },
]
\nextgroupplot
[
    grid = both, 
    grid style={dotted,draw=black!90},
    tick label style={/pgf/number format/fixed},
    xmode = linear, 
    ymode = linear, 
    ymax = 0.01, 
    ymin = 0.0, 
    xmax = 10, 
    xmin = 1,
    xtick = {1,2,3,4,5,6,7,8,9,10},
    ylabel = \large    attack accuracy,
    legend style =
        {
            legend columns=5,
            fill=none,
            draw=black,
            anchor=center,
            align=center
        },
    legend to name=acc_legend4
]

    \addplot[brown, ultra thick, solid, mark=pentagon*, mark options={fill=white} ,mark size=2.5pt] table [x index = {0}, y index={2}, col sep=comma]{./appendix_plots/csv_files_app/30_clients/MNIST_targeted_attack_acc_m_6.csv}; \addlegendentry{\footnotesize no defense}
    
    \addplot[red, ultra thick, mark=*, mark options={fill=white}, mark size=1.5pt] table [x index = {0}, y index={4}, col sep=comma]{./appendix_plots/csv_files_app/30_clients/MNIST_targeted_attack_acc_m_6.csv};
    \addlegendentry{\footnotesize {\NAME}-$\nmest$}

    \addplot[amber, ultra thick, mark=triangle*, mark options={fill=white}, mark size=2pt] table [x index = {0}, y index={3}, col sep=comma]{./appendix_plots/csv_files_app/30_clients/MNIST_targeted_attack_acc_m_6.csv};
    \addlegendentry{\footnotesize {\NAME}-$\Delta$}
  
    \addplot[darkblue, ultra thick, solid, mark=square*, mark options={fill=white}, mark size=1.5pt] table [x index = {0}, y index={1}, col sep=comma]{./appendix_plots/csv_files_app/30_clients/MNIST_targeted_attack_acc_m_6.csv};
    \addlegendentry{\footnotesize oracle}

    \addplot[chocolate, ultra thick, solid, mark=diamond*, mark options={fill=white} ,mark size=2.7pt] table [x index = {0}, y index={5}, col sep=comma]{./appendix_plots/csv_files_app/30_clients/MNIST_targeted_attack_acc_m_6.csv};
    \addlegendentry{RFA~\cite{Pillutla22}}

    \node[anchor=west] (source) at (axis cs:2,0.0047){\large    Group test performed};
    \node (destination) at (axis cs:1,0.0045){};
    \draw[->](source) to[bend right=20] (destination);
     
\nextgroupplot
    [
    grid = both, 
    grid style={dotted,draw=black!90},
    tick label style={/pgf/number format/fixed},
    xmode = linear, 
    ymode = linear, 
    ymax = 0.7, 
    ymin = 0.0, 
    xmax = 30, 
    xmin = 1,
    ytick = {0, 0.1, 0.2, 0.3, 0.4, 0.5, 0.6, 0.7},
    mark repeat=3,
    ylabel = \large attack accuracy,
    ]
    
     \addplot[brown, ultra thick, solid, mark=pentagon*, mark options={fill=white} ,mark size=2.5pt] table [x index = {0}, y index={2}, col sep=comma]{./appendix_plots/csv_files_app/30_clients/CIFAR10_targeted_attack_acc_m_6.csv}; 
    
    \addplot[red, ultra thick, mark=*, mark options={fill=white}, mark size=1.5pt] table [x index = {0}, y index={4}, col sep=comma]{./appendix_plots/csv_files_app/30_clients/CIFAR10_targeted_attack_acc_m_6.csv};

    \addplot[amber, ultra thick, mark=triangle*, mark options={fill=white}, mark size=2pt] table [x index = {0}, y index={3}, col sep=comma]{./appendix_plots/csv_files_app/30_clients/CIFAR10_targeted_attack_acc_m_6.csv};
  
    \addplot[darkblue, ultra thick, solid, mark=square*, mark options={fill=white}, mark size=1.5pt] table [x index = {0}, y index={1}, col sep=comma]{./appendix_plots/csv_files_app/30_clients/CIFAR10_targeted_attack_acc_m_6.csv};

    \addplot[chocolate, ultra thick, solid, mark=diamond*, mark options={fill=white} ,mark size=2.7pt] table [x index = {0}, y index={5}, col sep=comma]{./appendix_plots/csv_files_app/30_clients/CIFAR10_targeted_attack_acc_m_6.csv};

    \draw [dashed, thick] (5,1e-2) -- (5,1);
    \node[anchor=west] (source) at (axis cs:10,0.145){\large    Group test performed};
    \node (destination) at (axis cs:5,0.14){};
    \draw[->](source) to[bend right=20] (destination);

\nextgroupplot
[
    grid = both, 
    grid style={dotted,draw=black!90},
    tick label style={/pgf/number format/fixed},
    xmode = linear, 
    ymode = linear, 
    ymax = 0.9, 
    ymin = 0.75, 
    xmax = 10, 
    xmin = 1,
    xtick = {1,2,3,4,5,6,7,8,9,10},
    xlabel = \large    Communication rounds, 
    ylabel = \large    top-1 accuracy
]

    \addplot[brown, ultra thick, solid, mark=pentagon*, mark options={fill=white} ,mark size=2.5pt] table [x index = {0}, y index={2}, col sep=comma]{./appendix_plots/csv_files_app/30_clients/MNIST_untargeted_acc_vs_m_6.csv}; 
    
    \addplot[red, ultra thick, mark=*, mark options={fill=white}, mark size=1.5pt] table [x index = {0}, y index={4}, col sep=comma]{./appendix_plots/csv_files_app/30_clients/MNIST_untargeted_acc_vs_m_6.csv};

    \addplot[amber, ultra thick, mark=triangle*, mark options={fill=white}, mark size=2pt] table [x index = {0}, y index={3}, col sep=comma]{./appendix_plots/csv_files_app/30_clients/MNIST_untargeted_acc_vs_m_6.csv};
  
    \addplot[darkblue, ultra thick, solid, mark=square*, mark options={fill=white}, mark size=1.5pt] table [x index = {0}, y index={1}, col sep=comma]{./appendix_plots/csv_files_app/30_clients/MNIST_untargeted_acc_vs_m_6.csv};

    \addplot[chocolate, ultra thick, solid, mark=diamond*, mark options={fill=white} ,mark size=2.7pt] table [x index = {0}, y index={5}, col sep=comma]{./appendix_plots/csv_files_app/30_clients/MNIST_untargeted_acc_vs_m_6.csv};

    \node[anchor=west] (source) at (axis cs:2,0.77){\large    Group test performed};
    \node (destination) at (axis cs:1,0.75){};
    \draw[->](source) to[bend right=20] (destination);
     
\nextgroupplot
    [
    grid = both, 
    grid style={dotted,draw=black!90},
    tick label style={/pgf/number format/fixed},
    xmode = linear, 
    ymode = linear, 
    ymax = 0.8, 
    ymin = 0.1, 
    xmax = 30, 
    xmin = 1,
    ytick = {0, 0.1, 0.2, 0.3, 0.4, 0.5, 0.6, 0.7, 0.8},
    mark repeat=3,
    xlabel = \large    Communication rounds, 
    ylabel = \large    top-1 accuracy,
    ]
    
     \addplot[brown, ultra thick, solid, mark=pentagon*, mark options={fill=white} ,mark size=2.5pt] table [x index = {0}, y index={2}, col sep=comma]{./appendix_plots/csv_files_app/30_clients/CIFAR10_untargeted_acc_vs_m_6.csv}; 
    
    \addplot[red, ultra thick, mark=*, mark options={fill=white}, mark size=1.5pt] table [x index = {0}, y index={4}, col sep=comma]{./appendix_plots/csv_files_app/30_clients/CIFAR10_untargeted_acc_vs_m_6.csv};

    \addplot[amber, ultra thick, mark=triangle*, mark options={fill=white}, mark size=2pt] table [x index = {0}, y index={3}, col sep=comma]{./appendix_plots/csv_files_app/30_clients/CIFAR10_untargeted_acc_vs_m_6.csv};
  
    \addplot[darkblue, ultra thick, solid, mark=square*, mark options={fill=white}, mark size=1.5pt] table [x index = {0}, y index={1}, col sep=comma]{./appendix_plots/csv_files_app/30_clients/CIFAR10_untargeted_acc_vs_m_6.csv};

    \addplot[chocolate, ultra thick, solid, mark=diamond*, mark options={fill=white} ,mark size=2.7pt] table [x index = {0}, y index={5}, col sep=comma]{./appendix_plots/csv_files_app/30_clients/CIFAR10_untargeted_acc_vs_m_6.csv};

    \draw [dashed, thick] (5,1e-2) -- (5,1);
    \node[anchor=west] (source) at (axis cs:10,0.25){\large Group test performed};
    \node (destination) at (axis cs:5,0.25){};
    \draw[->](source) to[bend right=20] (destination);

\end{groupplot}
\node[above] at ([xshift=-4.5cm, yshift=5mm]group c2r1.north)
{\pgfplotslegendfromname{acc_legend4}};

\node[text width=6cm,align=center,anchor=north] at ([yshift=-12mm]group c1r2.south) {\large   Experiments on MNIST. \label{fig:mnist_n30}};

\node[text width=6cm,align=center,anchor=north] at ([yshift=-12mm]group c2r2.south) {\large   Experiments on CIFAR-10. \label{fig:cifar_n30}};

\end{tikzpicture}%
}
    \caption{Experimental results for a federated learning with $n=30$ clients out of which $\nm=6$ of them are malicious. The attack accuracy of a targeted attack is shown in row 1 and the top-1 accuracy of an untargeted attack strategy is shown in row 2. We show the performance of {\NAME} along with no defense, oracle, and RFA~\cite{Pillutla22}.}
    \label{fig:30clients}
    \vspace{-3ex}
\end{figure}

\bibliographystyle{ieeetr}
\bibliography{main}

\section*{Appendix\\ Details of the ISIC2019 Experiment}
\label{app:isic}

\begin{figure}[t]
    \centering
    \resizebox{.9\linewidth}{!}{\begin{tikzpicture}

    \def \threshold{20}

    \pgfplotsset{
        selective show sum on top/.style={
            /pgfplots/scatter/@post marker code/.append code={%
                \ifnum\coordindex=#1
                   \node[
                   at={(normalized axis cs:%
                       \pgfkeysvalueof{/data point/x},%
                       \pgfkeysvalueof{/data point/y})%
                   },
                   anchor=south,
                   ]
                   {\pgfmathprintnumber{\pgfkeysvalueof{/data point/y}}};
                \fi
            },
        },selective show sum on top/.default=0
    }

    \begin{axis}[
    ybar stacked, ymin=0, 
    bar width=8mm,
    symbolic x coords={1,2,3,4,5,6,7,8,9,10,11,12,13,14,15},
    xtick=data,
    nodes near coords,
    yticklabel=\empty,
    ytick style={draw=none},
    ylabel = \LARGE label count,
    xlabel = \LARGE client number,
    width=20cm,
nodes near coords,
        nodes near coords greater equal only/.style={
            small value/.style={
                /tikz/coordinate,
            },
            every node near coord/.append style={
                check for small values/.code={
                    \begingroup
                    \pgfkeys{/pgf/fpu}
                    \pgfmathparse{\pgfplotspointmeta<#1}
                    \global\let\result=\pgfmathresult
                    \endgroup
                    \pgfmathfloatcreate{1}{1.0}{0}
                    \let\ONE=\pgfmathresult
                    \ifx\result\ONE
                        \pgfkeysalso{/pgfplots/small value}
                    \fi
                },
                check for small values,
            },
        },
        nodes near coords greater equal only=25,
    ]

    \addplot [fill=blue!20] coordinates {
        ({1},272)
        ({2},291)
        ({3},292)
        ({4},8)
        ({5},60)
        ({6},189)
        ({7},301)
        ({8},323)
        ({9},328)
        ({10},10)
        ({11},274)
        ({12},291)
        ({13},289)
        ({14},327)
        ({15},319)};
    \addplot [fill=red!20] coordinates {
        ({1},470)
        ({2},686)
        ({3},780)
        ({4},1556)
        ({5},289)
        ({6},347)
        ({7},433)
        ({8},443)
        ({9},448)
        ({10},1557)
        ({11},776)
        ({12},454)
        ({13},454)
        ({14},445)
        ({15},449)};
    \addplot [fill=green!20] coordinates {
        ({1},309)
        ({2},247)
        ({3},86)
        ({4},2)
        ({5},3)
        ({6},1e-10)
        ({7},325)
        ({8},289)
        ({9},284)
        ({10},1e-10)
        ({11},91)
        ({12},300)
        ({13},326)
        ({14},291)
        ({15},282)};
    \addplot [fill=orange!20] coordinates {
        ({1},75)
        ({2},90)
        ({3},10)
        ({4},1e-10)
        ({5},1e-10)
        ({6},1e-10)
        ({7},82)
        ({8},90)
        ({9},75)
        ({10},1e-10)
        ({11},7)
        ({12},75)
        ({13},75)
        ({14},70)
        ({15},73)};
    \addplot [fill=cyan!20] coordinates {
        ({1},129)
        ({2},410)
        ({3},190)
        ({4},60)
        ({5},10)
        ({6},165)
        ({7},116)
        ({8},122)
        ({9},111)
        ({10},53)
        ({11},210)
        ({12},134)
        ({13},120)
        ({14},121)
        ({15},129)};
    \addplot [fill=pink!20] coordinates {
        ({1},15)
        ({2},25)
        ({3},29)
        ({4},13)
        ({5},4)
        ({6},1e-10)
        ({7},12)
        ({8},9)
        ({9},17)
        ({10},14)
        ({11},21)
        ({12},11)
        ({13},11)
        ({14},12)
        ({15},18)};
    \addplot [fill=magenta!20] coordinates {
        ({1},10)
        ({2},1)
        ({3},34)
        ({4},20)
        ({5},2)
        ({6},1e-10)
        ({7},12)
        ({8},12)
        ({9},9)
        ({10},25)
        ({11},39)
        ({12},18)
        ({13},10)
        ({14},16)
        ({15},11)};
    \addplot [fill=gray!20, selective show sum on top/.list={0,1,2,3,4,5,6,7,8,9,10,11,12,13,14,15}] coordinates {
        ({1},47)
        ({2},156)
        ({3},3)
        ({4},1e-10)
        ({5},1e-10)
        ({6},1e-10)
        ({7},46)
        ({8},39)
        ({9},55)
        ({10},1e-10)
        ({11},7)
        ({12},45)
        ({13},42)
        ({14},45)
        ({15},46)};
    \legend{0,1,2,3,4,5,6,7}
    \end{axis}
 \end{tikzpicture}}
    \vspace{-1ex}
    \caption{Client data partitions for experiments over the ISIC2019 dataset.}
    \label{fig:ISICpartition}
    \vspace{-3ex}
\end{figure}

The ISIC2019 dataset~\cite{Codella} is a public dataset consisting of images of various skin lesion types, including malignant melanomas and benign moles, used for research in dermatology and skin cancer detection. We use the ISIC2019 dataset and follow~\cite{Ogier2022flamby} by first restricting our usage to 23247 data samples (out of 25331 entries) from the public dataset due to metadata availability and then preprocessing by resizing the shorter side to $224$ pixels and by normalizing the contrast and brightness of the images.
Next, we randomly divide the data  into a test and a training dataset  of size $3388$ and $19859$, respectively.
The server validation set is created by randomly sampling $100$ data entries from the dataset.
Next, as in~\cite{Ogier2022flamby}, the remaining $19759$  samples are partitioned into $6$ partitions with respect to the image acquisition system used.
The $6$ partitions are then split into $15$ partitions by iteratively splitting the largest partition in half.
This procedure results in partitions with heterogeneity in the number of data samples and label distribution (see Fig.~\ref{fig:ISICpartition}), and in the feature distribution due to different acquisition systems (see~\cite[Fig.~1.f]{Ogier2022flamby}).

Due to the large imbalance in the dataset (label $1$ corresponds to $48.7$\% whereas label $5$ and $6$ are represented by about $1$\% of the entries), the focal loss is used in the training~\cite{Lin2017focal} and we use the balanced accuracy to assess the performance of the trained network.
Furthermore, to encourage generalization during training, we follow~\cite[App.~H]{Ogier2022flamby} and apply random augmentations to the training data.

Finally, the heterogeneous data partitioning causes the choice of malicious clients to significantly impact the outcome of the experiment.
For this reason, we let the set of malicious clients  $\mathcal{M}_{i} \subset \mathcal{M}_{j}$ for $j>i$ to ensure that results across different values of $\nm$ are comparable. %

\end{document}